\documentclass[11pt]{article}

\usepackage{graphicx}
\usepackage{color}
\usepackage{algorithmic,algorithm}
\usepackage{dsfont}
\usepackage{natbib}
\usepackage{amsmath,amsfonts,amsthm,amssymb,amscd}
\usepackage{fullpage}
\usepackage{hyperref}

\newcommand{\R}{\mathbb{R}}
\renewcommand{\P}{\mathbb{P}}

\newcommand{\N}{\mathbb{N}}
\newcommand{\E}{\mathbb{E}}

\newcommand{\Ups}{\Upsilon}

\DeclareMathOperator{\diag}{diag}

\DeclareMathOperator{\shape}{Shape}

\DeclareMathOperator*{\argmin}{argmin}
\DeclareMathOperator*{\argmax}{argmax}

\DeclareMathOperator*{\new}{new}

\newcommand{\aic}{\text{aic}}
\newcommand{\bic}{\text{bic}}

\newcommand{\1}{{\rm 1}\kern-0.24em{\rm I}}

\newcommand{\bX}{\mathbf X}
\newcommand{\bbX}{\mathbb X}
\newcommand{\bbZ}{\mathbb Z}

\newcommand \cC{{\cal C}}
\newcommand \cD{{\cal D}}

\newcommand \cH{{\cal H}}

\newcommand \cK{{\cal K}}

\newcommand{\cP}{{\mathcal P}}

\newcommand \cX{{\cal X}}

\newcommand{\var}{\text{Var}}%
\newcommand{\varb}{\var_{\text{B}}}%
\newcommand{\norm}[1]{\|#1\|}%

\newcommand{\kl}{\mathcal K}

\newcommand{\ind}[1]{\mathbf 1_{#1}}

\newtheorem{Theorem}{Theorem}[section]
\newtheorem{Lemma}[Theorem]{Lemma}

\begin{document}

\title{Sparse Bayesian Unsupervised Learning}


\author{St\'ephane Ga{\"i}ffas \thanks{CMAP -- Ecole
    Polytechnique. Email:
    \url{stephane.gaiffas@cmap.polytechnique.fr}} \and Bertrand Michel
  \thanks{Universit\'e Pierre et Marie Curie, Paris~6. Email:
    \url{bertrand.michel@upmc.fr}}}

\date{\today}

\maketitle

\begin{abstract}
  This paper is about variable selection, clustering and estimation in
  an unsupervised high-dimensional setting. Our approach is based on
  fitting constrained Gaussian mixture models, where we learn the
  number of clusters $K$ and the set of relevant variables $S$ using a
  generalized Bayesian posterior with a sparsity inducing prior. We
  prove a sparsity oracle inequality which shows that this procedure
  selects the optimal parameters $K$ and $S$. This procedure is
  implemented using a Metropolis-Hastings algorithm, based on a
  clustering-oriented greedy proposal, which makes the convergence to
  the posterior very fast.
\end{abstract}

\section{Introduction}
\label{sec:introduction}

This paper is about variable selection, clustering and estimation for
case where we observe unlabelled i.i.d data $X_1, \ldots, X_{n}$ in
$\R^d$, denoted $X_i = (X_i^1, \ldots, X_i^d)$, with $d$ being
eventually much larger than $n$. Clustering is now an important tool
for the analysis of high-dimensional data. An example of application
is gene function discovery and cancer subtype discovery, where one
wants to construct groups of genes with their expression levels across
different conditions or across several patients tissue samples,
see~\cite{eisen1998cluster} and~\cite{golub1999molecular} for
instance.

In the high-dimensional setting, clustering becomes challenging
because of the presence of a large number of noise variables, that can
hide the cluster structure. So, one must come up with an algorithm
that, at the same time, selects relevant variables and constructs a
clustering based only on these variables. This interconnection between
variable selection and clustering makes the problem challenging, and
contrasts with supervised problems.

Among many approaches, model-based clustering becomes increasingly
popular. It benefits from a well-understood probabilistic framework,
see \cite{Dempster:77}, \cite{Banfield:93}, \cite{Celeux:95},
\cite{MR1951635}. Because of its flexibility and interpretability, the
most popular is the Gaussian Mixture Model (GMM). Indeed, GMM can be
classified into~28 models \citep{Banfield:93,Celeux:95}, where various
constraints on the covariance matrices allow to control the shapes and
orientations of the clusters. Variable selection in this context is
based on the decomposition $\mu_k = \bar \mu + m_k$ of the mean of the
$k$-th component of the GMM, where $\bar \mu$ is a global mean and
$m_k$ contains specific information about the $k$-th cluster. A
sparsity assumption on the $m_k$'s is natural in this setting: it
means that only a few variables contribute to the clustering in terms
of location. Even further, since the clustering is totally determined
by the estimation of the GMM on the relevant variables, it is natural
to assume that the $m_k$'s share the \emph{same} support in order to
get an accurate description of these variables.

There are roughly two approaches for variable selection in the context
of high-dimensional model-based clustering: the Bayesian approach, see
\cite{LiuEtAl03}, \cite{Hoff05}, \cite{Hoff06}, \cite{MR2160563},
\cite{RafteryDean2006}, \cite{MR2285077} and the penalization
approach, see \cite{pan2007penalized}, 
\cite{ZhouPanShen09}. The Bayesian approach is very flexible and
allows complex modelings of the covariance matrices for the
GMM. However, this approach is computationally demanding since it
requires heavy MCMC stochastic search on continuous parameter
sets. The penalization approach is based on penalizing the
log-likelihood by the $\ell_1$-norm of the mean vectors and inverse
covariance matrices of the GMM. It leads to a soft-thresholding
operation in the M-step of the Expectation-Maximization (EM)
algorithm, see~\cite{Dempster:77}, \cite{pan2007penalized}. By doing
so, many coordinates of the mean vectors are shrunk towards zero,
which helps, hopefully, to remove noise variables. A problem with this
thresholding (or equivalently $\ell_1$-penalization) approach is that
the resulting estimated mean vectors have no reason to share the same
supports. As mentioned above, this property is strongly suitable since
it models precisely the fact that only a few variables makes a
distinction across the clusters.
Moreover, another problem is that, as far as we know, there exists no
mathematical result about the statistical properties of these
penalized GMM methods for high-dimensional unsupervised learning, such
as upper bounds on the estimation error, or sparsity oracle
inequalities (for mixture models in the regression setting, one can
see~\cite{MR2677722} and~\cite{Meynet12}).

The motivations for this work are two-fold. First, we propose a new
approach for model-based unsupervised learning, by combining the GMM
with a learning procedure based on the PAC-Bayesian approach. The
PAC-Bayesian approach was originally developed for classification by
\cite{shawe1997pac}, \cite{mcallester1998some} and
\cite{Catoni04,Catoni07}, see also \cite{PhD04a,audibert-2009-37},
\cite{MR2483458} and \cite{Zhang06a,Zhang06b} for further
developments. This approach has proved successful for sparse
regression problems,
see~\cite{dalalyan_tsybakov07,dalalyan2008aggregation,dalalyan2009sparse},
\cite{RigTsy11ss}, \cite{MR2786484}, \cite{AlquierBiau11}. In this
work, we use prior distributions that suggest a small support $S
\subset \{ 1, \ldots, d \}$ for significant variables and a small
number $K$ of clusters. Then, we learn $S$ and $K$ using a randomized
aggregation rule based on the Gibbs posterior distribution, see
Section~\ref{sec:learning_K_and_S}. Our methodology is based on a
Metropolis-Hastings (MH) exploration algorithm that explores a
discrete (but large) set for the ``meta'' parameter $\eta = (K,
S)$. The exploration can be done in a very efficient way, thanks to
the use of a proposal which is particularly relevant for the
clustering problem, see Section~\ref{sec:implementation}. As shown in
our empirical study (see Sections~\ref{sec:implementation}
and~\ref{sec:numerical-experiments}), an order of $300$ steps in the
MH algorithm is sufficient for convergence on a large scale problem.

Second, our methodology is supported by strong theoretical guarantees:
using PAC-Bayesian tools~\cite{Zhang06a,Zhang06b,Catoni04,Catoni07},
we prove a sparsity oracle inequality for our procedure. This oracle
inequality shows that our procedure automatically selects the
parameter $\eta = (K, S)$ which is optimal in terms of estimation
error. Note that this is the first result of this kind for sparse
model-based unsupervised learning.

\section{Our procedure}

We assume from now on that we have $n$ observations. First, we split at random the whole sample $\bbX =
(X_1, \ldots, X_{n})$ into a \emph{learning} sample $\bbX_1$ and an
\emph{estimation} sample $\bbX_2$, of sizes $n_1$ and $n_2$ such that $n_1+n_2=n$. Then, the two main
ingredients of our procedure are fitting constrained Gaussian Mixture
Models (GMM) and learning the number of clusters and relevant
variables using a generalized Bayesian Posterior with a sparsity
inducing prior, which are described respectively in
Sections~\ref{sec:constrained-GMM} and~\ref{sec:learning_K_and_S}. The
main steps of our procedure are finally summarized in
Section~\ref{sec:main-steps}.

\subsection{Constrained Gaussian mixtures models}
\label{sec:constrained-GMM}

Let us denote by $\phi_{(d)}(\cdot |\mu_k, \Sigma_k)$ the density of the 
multivariate Gaussian $N_{d}(\mu_k, \Sigma_k)$ distribution. The density
of a Gaussian mixture model (GMM) with $K$ components is given by
\begin{equation*}
  f_\theta = \sum_{k=1}^K p_k \ \phi_{(d)}(\cdot |\mu_k, \Sigma_k),
\end{equation*}
where $p_1, \ldots, p_K$ are the mixture proportions, $\mu_1, \ldots,
\mu_K$ are the mean vectors and $\Sigma_1, \ldots, \Sigma_K$ are the
covariance matrices. Let $\theta = (p_1,\dots, p_K,\mu_1, \dots,
\mu_K, \Sigma_1, \dots, \Sigma_K)$ be the parameter vector of the
GMM. An interest of GMM is that once the model is fitted, one can
easily obtain a clustering using the Maximum A Posteriori (MAP) rule,
which is recalled in Section~\ref{sec:em-map} below.
As explained in Introduction, the following structure on the $\mu_k$'s
is natural:
\begin{equation*}
  \mu_k = \bar \mu + m_k,
\end{equation*}
where $\bar \mu$ is the global mean and where $m_1, \ldots, m_K$ are
sparse vectors that share the \emph{same} support. The idea behind
this structure is that we want only a few variables to have an impact
on the clustering, and even further, we want an informative variable
to be informative for \emph{all} clusters. In the following it is
assumed without loss of generality that the data have been centered
and normalized so that we can take $\bar \mu = 0$. Note that this is a
classical assumption in the model-based clustering context.





Let $\cP(A)$ denote the set of all the subsets of a finite set $A$,
and $|A|$ denote the cardinality of $A$. Let us introduce the set of
GMM configurations
\begin{equation*}
  \Ups = \N^* \times \cP(\{ 1, \ldots, d \}).
\end{equation*}
A parameter $\eta = (K, S) \in \Ups$ is a ``meta-parameter'' for
$\theta$. It fixes the number of clusters $K$ and the common support
$S$ of the vectors $m_k$. The set $S$ is thus the set of indexes of
the active variables. If $\mu \in \R^d$ and $S \subset \{ 1, \ldots,
d\}$ we define $\mu_S = (\mu_j)_{j \in S} \in \R^{|S|}$. If $\Sigma$
is a $d \times d$ matrix, we also define the $|S| \times |S|$ matrix
$\Sigma_S = (\Sigma_{j, j'})_{j,j' \in S}$. 

Let $\shape(\eta)$ denotes the shape of the GMM restricted to the
active variables, namely $\shape(\eta)$ is a particular set of
constrained $K$-vector of $|S| \times |S|$ covariance matrices. An
example is, using the notation introduced in~\cite{Celeux:95}, the
shape
\begin{equation}
  \label{eq:shape-lb}
  \begin{split}
    \shape_{LB}(\eta) = \Big\{ \left( (\Sigma_1)_S, \cdots , (\Sigma_K)_S \right)
    \; : \; (& \Sigma_1)_S = \cdots = (\Sigma_K)_S =
    \diag(\sigma_1^2,\dots, \sigma^2_{|S|}), \\
    \;\; &(\sigma_1,\dots, \sigma_{|S|}) \in (\R^+)^{|S| } \Big\},
  \end{split}
\end{equation}
which corresponds to identical and diagonal matrices with normalized
noise variables. Note that in the high-dimensional setting, a simple
structure on noise variables is suitable, since fitting a GMM with
large covariance matrices is prohibitive. The theoretical result given
in~Section~\ref{sec:SparseIneq} is valid for all possibles
shapes. Even more than that, since our approach is based on an
aggregation algorithm, one could perfectly mix several GMM fits with
different shapes (leading to an increased computational cost). From
now on, we fix a family of shapes $(\shape(\eta))_{\eta \in \Ups
}$. We also assume that active variables and non-active variables are
independent. This is only for the sake of simplicity, since we could
use more elaborated models, such as the ones
from~\cite{MaugisCeleuxMartin09}. All these assumptions lead to the
following set of parameters associated to a configuration $\eta$ and a
shape $\shape(\eta)$:
\begin{align*}
  \Theta \left(\eta,\shape(\eta) \right) = \Big\{ & (p_1,\dots,p_K,\mu_1, \dots, \mu_K,
  \Sigma_1, \dots, \Sigma_K) \; : \; (\mu_1)_{S^\complement} =
  \dots = (\mu_K)_{S^\complement} = 0, \\
  &\quad  \left( (\Sigma_1)_S, \dots , (\Sigma_K)_S \right)
  \in \shape(\eta), \quad (\Sigma_1)_{S^\complement} = \cdots =
  (\Sigma_K)_{S^\complement} = I_{d-|S^\complement|} \Big\},
\end{align*}
where $(\mu_k)_A$ is the projection of $\mu_k$ on the set of coordinates $A$, where $(\Sigma_k)_A$ is the restriction of the matrix $\Sigma_k$ on
$A\times A$ and where $I_q$ stands for the identity matrix on $\R^q$.  Then, for $\theta \in \Theta \left(\eta,\shape(\eta) \right)$, the density
$f_\theta$ can be decomposed as follows :
\begin{equation}
  \label{eq:decompftheta}
f_\theta = \phi_{(d-|S|)}\left(\cdot |0_{d-|S|}, I_{d-|S|}\right) \: \sum_{k=1}^K p_k \phi_{(S)}(\cdot | (\mu_k)_S, (\Sigma_k)_S).
\end{equation}
For instance, if $\eta = (K, S)$ and $\theta \in \Theta\left(\eta,\shape_{LB}\right)$, then $f_\theta$ is a
Gaussian mixture with $K$ components, with variables outside of $S$
that are uncorrelated and standard $N(0, 1)$ and a shape  $\shape_{LB}$.

We consider the maximum likelihood estimator of $\theta$ over the set
of constraints $\Theta(\eta)$ for the observations in the estimation
sample $\bbX_2$:
\begin{equation}
  \label{eq:theta-eta-def}
  \hat \theta\left(\eta,\shape(\eta)\right) \in \argmin_{\theta \in  \Theta \left(\eta,\shape(\eta) \right)}
  L_\theta(\bbX_2) \;\; \text{ where } \;\; L_\theta(\bbX_2) =
  \sum_{i : X_i \in \bbX_2}  - \ln f_\theta(X_i).
\end{equation}
An approximation of this estimator can be computed using the
Expectation-Maximization (EM) algorithm, see Section~\ref{sec:em-map}
for more details. In the following we assume that a shape has been fixed and then use the notation $\Theta(\eta)$ and $\hat \theta(\eta)$.

\subsection{Learning $K$ and $S$}
\label{sec:learning_K_and_S}

Now, we want to learn the number of clusters $K$ and the support $S$
based on the learning data $\bbX_1$. We use a randomized aggregation
procedure, which corresponds in this setting to a generalized Bayesian
posterior. This method relies on the choice of prior
distributions. For $K$, we have in mind to force the number of
clusters to remain small, so we simply consider the Poisson prior
\begin{equation*}
  \pi_{\text{clust}}(K) = \frac{e^{-1}}{K!} \ind{K \geq 0}.
\end{equation*}
This prior gives nice results and always recovers smoothly the correct
number of clusters in most settings. The use of another intensity
parameter (we simply take $1$ here) does not make significant
differences. This comes from the fact that the parameter $\lambda$ of
the generalized Bayesian posterior (see~\eqref{eq:gibbs-posterior}
below) already tunes the degree of sparsity over the whole learning
process.

As explained above, we want only the most significant variables to
have an impact on the final clustering. So, we consider a prior that
downweights supports with a large cardinality exponentially
fast. Namely, we consider
\begin{equation*}
  \pi_{\text{supp}}(S) = \frac{1}{\binom{d}{|S|} e^{|S|} C_d},
\end{equation*}
where $C_d = \sum_{k=0}^d e^{-k}$. This prior is used in
\cite{RigTsy11ss} for sparse regression learning. This choice is far
from being the only one, since we observed empirically that any other
prior inducing a small cardinality gives similar results. Finally, the
prior for a meta-parameter $\eta = (K, S) \in \Ups$ is
\begin{equation*}
  \pi(\eta) = \pi_{\text{clust}}(K) \times \pi_{\text{supp}}(S).
\end{equation*}
From a Bayesian point of view, this means that we assume $K$ and $S$
to be independent, which is reasonable in this setting. This choice has an impact on the risk bound we obtain, this is discussed further.
 The Gibbs
posterior distribution is now defined by
\begin{equation}
  \label{eq:gibbs-posterior}
  \hat \pi_\lambda(d\eta) = \frac{\exp\Big(- \lambda L_{\hat
      \theta(\eta)}(\bbX_1) \Big)}{\E_{\eta \sim \pi}
    \exp\Big(- \lambda L_{\hat
      \theta(\eta)}(\bbX_1) \Big)} \pi(d \eta),
\end{equation}
where $\lambda > 0$ is the \emph{temperature} parameter and where we
recall that $\hat \theta(\eta)$ is given by~\eqref{eq:theta-eta-def}
and that $L_{\theta}(\bbX_1) = \sum_{i : X_i \in \bbX_1}  \ln
f_\theta(X_i)$. Note that written in the following way:
\begin{equation}
  \label{eq:generalized-bayesian-posterior}
  \hat \pi_\lambda(d\eta) = \frac{\prod_{i : X_i \in \bbX_1} f_{\hat
      \theta(\eta)}^\lambda(X_i)}{\E_{\eta \sim \pi} \prod_{i : X_i
      \in \bbX_1} f_{\hat \theta(\eta)}^\lambda(X_i)} \pi(d \eta),
\end{equation}
the Gibbs posterior is the Bayesian posterior when $\lambda = 1$, and
it is a so-called generalized Bayesian posterior otherwise. We use a
Metropolis-Hastings algorithm to pick at random $\eta$ with
distribution $\hat \pi_\lambda$, see Section~\ref{Sec:proposal}.
The final randomized
aggregated estimator is given by
\begin{equation}
  \label{eq:aggregated-estimator}
  f_{\hat \theta(\eta)} \;\; \text{ with } \;\; \eta \sim \hat
  \pi_\lambda.
\end{equation}
The parameter $\lambda$ is a smoothing parameter that makes the
balance between goodness-of-fit on $\bbX_1$ and the Kullback-Leibler
divergence to the prior (see
Equation~\eqref{eq:characterization-posterior} below). Hence, a
careful data-driven choice for $\lambda$ is important, we observed
that AIC or BIC criteria gives satisfying results, see
Section~\ref{subsec:PostTrait}.

\subsection{Main steps}
\label{sec:main-steps}

The main steps of our procedure can be summarized as follow. We fix a
number $B$ of splits (say 20).
\begin{enumerate}
\item Repeat the following $B$ times:
  \begin{enumerate}
  \item Split the whole sample $\bbX = (X_1, \ldots, X_{n})$ at
    random into a learning sample $\bbX_1$ and an estimation sample
    $\bbX_2$, of size $n_1$ and $n_2$. We call $b$ this split in the
    following.
  \item For each $\lambda$ in a grid $\Lambda$, use the
    Metropolis-Hastings (MH) algorithm described in
    Section~\ref{Sec:proposal} below to pick at random $\eta = (K, S)$
    distributed according to $\hat \pi_\lambda$,
    see~\eqref{eq:generalized-bayesian-posterior}. Along the MH
    exploration, the likelihoods used in the generalized Bayes
    posterior are computed using $\bbX_1$ whereas, to compute
    approximations of the $\hat \theta(\eta)$'s, the EM algorithm uses
    $\bbX_2$.
  \item Select a temperature $\lambda(b)$ from the grid, see
    Section~\ref{subsec:PostTrait}.
  \item Choose a configuration $\eta(b)$ for temperature $\lambda(b)$
    and select a final meta-parameter $\hat \eta$ using the $\eta(b)$
    chosen by each split $b = 1, \ldots, B$, see
    Section~\ref{subsec:PostTrait}. Fit on $\bbX$ a final GMM with
    configuration $\hat \eta$.
  \end{enumerate}
\item Use the MAP rule to obtain a final clustering.
\end{enumerate}

Several alternatives to these steps are possible since the $B$ splits
bring a lot of information that can be analyzed in different ways. For
instance an ``aggregated clustering'' based on a proximity matrix
summarizing all the links pointed out by the $B$ clusterings can be
easily computed (see Section \ref{subsec:PostTrait}).

\section{Main results}
\label{sec:SparseIneq}

If $f$ and $g$ are probability density functions, we denote
respectively by $\cH(f, g)$ and $\cK(f, g)$ the Hellinger and Kullback
divergences between $f$ and $g$. We denote by $\E_{\bbX_1}$ and  $\E_{\bbX_2 }$ the
expectations with respect to the learning sample $\bbX_1$ and $\bbX_2$. If the samples
$\bbX_1$  and $\bbX_2$ are i.i.d with a common density $f^*$, a measure of statistical
risk for~\eqref{eq:aggregated-estimator} is simply given by
$\E_{\bbX_1} \E_{\eta \sim \hat \pi_\lambda} \cH^2(f^*, f_{\hat
  \theta(\eta)})$. The next Theorem is a sparsity oracle inequality
for the aggregated estimator~\eqref{eq:aggregated-estimator}.
\begin{Theorem}
  \label{thm:main-oracle-inequality}
  Let $\lambda \in (0, 1)$ and consider the randomized aggregated
  estimator~\eqref{eq:aggregated-estimator}, where we recall that
  $\hat \theta(K, S)$ is a constrained maximum likelihood
  estimator~\eqref{eq:theta-eta-def} for a fixed shape. Conditionally to $\bbX_2$ we have
  \begin{eqnarray*}
    \E_{\bbX_1}  \cH^2(f^*, \E_{\eta \sim \hat
\pi_\lambda} f_{\hat \theta(\eta)})  
& \leq & \E_{\bbX_1} \E_{\eta \sim \hat \pi_\lambda}  \cH^2(f^*,  f_{\hat \theta(\eta)})      \\
  &  \leq &   c_\lambda \inf_{\substack{K \in \N^* \\ S   \subset \{ 1, \ldots, d \}}} 
    \bigg\{ \lambda \kl(f^*,  f_{\hat \theta(K, S)}) + \frac{ \ln K! + 1 + 2 |S|  \ln ( e  d / |S|) )}{n} \bigg\} 
  \end{eqnarray*}
and furthermore
  \begin{eqnarray}
    \E_{\bbX_1 \bbX_2}  \cH^2(f^*, \E_{\eta \sim \hat \pi_\lambda} f_{\hat \theta(\eta)})
 & \leq  &   \E_{\bbX_1 \bbX_2} \E_{\eta \sim \hat \pi_\lambda} \cH^2(f^*,      f_{\hat \theta(\eta)})    \notag  \\
 & \leq& c_\lambda  \inf_{\substack{K \in \N^* \\ S  \subset \{ 1, \ldots, d \}}} \bigg\{ \lambda  \E_{\bbX_2} \kl(f^*,  f_{\hat \theta(K, S)}) +
\frac{ \ln K! + 1 + 2 |S|  \ln ( e  d / |S|) )}{n} \bigg\}    \label{eq:borneTH1}
  \end{eqnarray}
where $c_\lambda = 2 / \min(\lambda, 1-\lambda)$.
\end{Theorem}
The proof of Theorem~\ref{thm:main-oracle-inequality} is given in
Section~\ref{sec:proof-thm1}
below. Theorem~\ref{thm:main-oracle-inequality} entails that the
aggregated estimator~\eqref{eq:aggregated-estimator} has a estimation
error close to the one of the GMM with the best number of clusters $K$
and the best support $S$. Hence, it proves that the generalized
Bayesian posterior~\eqref{eq:generalized-bayesian-posterior}
automatically selects the correct $K$ and $S$ in terms of estimation
error. Note that the residual term is of order $|S|  \ln(d / |S|) /
n$, which coincides with the optimal residual term for the
model-selection aggregation problem in supervised settings, see
\cite{Tsy2003}, \cite{BTW07}, \cite{LecBer}. Broadly, this computational cost is $
|S| \ln |S| + K \ln d $, this additive structure additive in $|S|$
and $K$ is a direct consequence of the assumption of independence between $S$ and $K$ we make in the prior. The use of the squared Hellinger
distance on the left hand side and of the Kullback divergence on the
right hand side is a standard technical problem for the proof of
oracle inequalities for density estimation, see for
instance~\cite{van_de_geer00}, \cite{Zhang06b}, and
\cite{massart03}. The restriction on $\lambda$ is for the sake of
simplicity, since more general (but sub-optimal) bounds for $\lambda >
1$ can be established, see \cite{Zhang06b}, and requires extra
technicalities that are beyond the scope of this paper.

The upper bound given in Theorem~\ref{thm:main-oracle-inequality} can be made more explicit by computing the Kullback Leibler risk of the
MLE in a family of GMMs of fixed shape, we investigate this question in following of this section.
 It is known since the works are \cite{Akaike73} (see also Section 2.4 in
\cite{Linhart86}) that for well
specified parametric models in $\R^p$, the $\kl$ risk of the MLE is of the order of $\frac p n$. However, the model has no reason to be well
specified in our context and moreover we need a more precise bound than these asymptotic results. It is also well known (see for instance
\cite{Birge83}) that rates of convergence for estimators should be related to the metric structure with $\cH$. Indeed, rates of convergence of the MLE
are usually given in term of Hellinger risk. In the context of GMMs, rates of convergence of MLE for the Hellinger risk have first been investigated
in \cite{Genovese:00} and
\cite{Ghosal:01}. To compute such rates of convergence for our models, we need to bound the parameters of the sets $ \Theta(\eta)$. We assume as
before that the GMM
shape is fixed.  For a configuration $\eta$ and some positive constants  $\bar{\mu} $,  $\underline{\sigma} < 1 < \bar{\sigma}$, $\underline{L}   < 
\bar{L}$  we consider restricted sets of parameters with the following constraints:
$$
\Theta_r(\eta)  = \Big\{ \theta \in  \Theta(\eta) \, : \,   \forall k \in  \{1 \dots K \}  \; , \;  |\mu_k|   \leq    \bar{\mu}    \, , \,  
\mbox{sp}(\Sigma_k) \subset [ \underline{\sigma}^2 \, , \, \bar{\sigma }^2 ]^d  \, , \, \underline{L}   \leq    (2\pi)^{d/2} |\Sigma_k | \leq \bar{L}
         \Big\}
$$
where $\mbox{sp}(\Sigma)$  is the spectrum of $\Sigma$. Let $\mathcal F _\eta$ be the set of GMM densities $f_\theta$ for $\theta \in \Theta_r(\eta)
$. Let $\hat{\theta}_r(\eta)$ denotes the MLE of $\theta$ in the restricted set $\Theta_c(\eta)$. We then consider the randomized restricted
estimator 
$$f_{\hat \theta_r(\eta)} \quad \textrm{ with  } \quad \eta \sim \hat \pi _{\lambda}.$$
We also define 
$$\kl(f^*,\mathcal F _\eta) := \inf _{\theta \in \Theta_r(\eta)} \kl(f^*, f_{\theta}) $$
and  the number of free parameters $D(\eta)$ of the Gaussian mixture model parametrized by $ \Theta_r(\eta)$. Note that this number of free parameters
depends of the chosen shape. For instance, if the shape is $\shape_{LB}$ then $D(\eta) =  (K+1) |S| -1$.

In order to upper bound the $\kl$-risk, we use the following standard result: for any measure $P$ and $Q$ defined on the same
probability space (see Section 7.6 in \cite{massart03} for instance),
 \begin{equation} \label{eq:klHelling}
  2 \cH^2 (P,Q)  \leq \kl(P,Q) \leq  2 \left( 1 + \ln \left\| \frac {d P }{d Q }  \right\|_\infty   \right)  \cH^2(P,Q) .
 \end{equation}
To avoid "unbounded problem"  we
will assume here that $f^\star$ is bounded and has a compact support. Under these hypotheses it is then possible to upper bound the $\kl$-risk of the
MLE on the spaces defined $\Theta_r(\eta)$.

\begin{Theorem} \label{theo:BorneSupAvecMLE}
Assume that $\bbX_1$ and $\bbX_2$ are both i.i.d.  with a common density $f^*$ such that $\| f^\star \| _\infty < \infty$ and such that the
support of $f ^\star < \infty $ is included in $B(0,\bar \mu)$.  There exists an absolute constant $\kappa$ such that for any $\lambda \in (0, 1)$,
\begin{multline*}
    \E_{\bbX_1 \bbX_2} \cH^2(f^*,  \E_{\eta \sim \hat \pi_\lambda}    f_{\hat \theta(\eta)}) \leq  \\
    c_\lambda  \inf_{\substack{K \in \N^* \\ S   \subset \{ 1, \ldots, d \}}}  
\bigg\{  
  \lambda C 
   \left[ 
    \kl (f ^\star, \mathcal F_\eta) + \kappa    \frac{D(\eta)}{n_2}   \left\{ \mathcal{A} ^2 +  \ln^+ \left( \frac {n_2} {\mathcal{A} ^2  
D(\eta)} \right) \right\} + \frac \kappa {n_2}  \right]      
  + \frac{ \ln K! + 1 + 2 |S|  \ln ( e  d / |S|) )}{n_1} 
\bigg\}
\end{multline*}
where 
$C  =   \frac {8   }  {2 \ln 2  -   1}   \left( 1 + \ln ( \|f^\star\|_\infty   L ^ +)    +   \frac {2 {\bar \mu}^2} {\overline{\sigma} ^2 } 
\right)$ and where the constant $\mathcal{A}$ only depends on the GMM shape and the bounding parameters.
\end{Theorem}
Of course, Theorem~\ref{theo:BorneSupAvecMLE} is only meaningful if the true distribution $f^*$ can be arbitrarily well-approximated in the
$\kl$ divergence sense by the Gaussian mixtures of the model collection. Roughly, the model dimension $D(\eta)$ is of the order of $ K |S|$  or  $ K
|S|^2$ according to the chose shape. The upper bound given by this result is thus of the order of 
$$ \inf_{\substack{K \in \N^* \\ S    \subset \{ 1, \ldots, d \}}} 
       \bigg\{ \kl (  f^\star,\mathcal F _\eta) +  \frac{D( \eta)}{n } \ln \frac {n }{D(\eta)} \bigg\} $$ 
if we take for instance $n_1 = n_2$. A similar bound has been found in the same framework by \cite{MaugisMichel11a} for an $l_0$-penalization
procedure in the spirit of \cite{massart03}. Moreover, using recent results of \cite{KRV} about the approximation of log-Holder densities using
univariate Gaussian mixtures models, \cite{MaugisMichel2012} has shown the optimality of such risk bound, in the minimax sense. 

Note that the assumptions on the true density $f^\star$  could be probably relaxed. However these assumptions are not too strong for the clustering
framework of this paper. Finally, it is not possible to give one simple expression for the constant $\mathcal A$ since it depends on the shape
chosen. The interested reader is referred to Lemma~\ref{lem:EntropyGMM} in the Appendix and references therein.

\section{Implementation}
\label{sec:implementation}

This section details the whole implementation of our method. We
illustrate the main steps on the a simulated example presented
further.
 
\subsection{The EM algorithm and the MAP rule}
\label{sec:em-map}

An approximation of the maximum likelihood estimator of GMMs can be
computed thanks to the EM algorithm \citep{Dempster:77}. In this
section we briefly present the principle of the algorithm for our
context. Assuming that a split $(\bbX_1, \bbX_2)$ has been chosen,
only the estimation sample $\bbX_2$ can be used to compute the
estimators. Let $\hat \theta(\eta)$ be the maximum likelihood
estimator of the GMM with configuration $\eta$, see
Equation~\eqref{eq:theta-eta-def}. Let $Z_i$ be the (unknown) random
vector giving the cluster of the observation $i$:
\begin{equation*}
  Z_{i,k} =
  \begin{cases}
    \; 1  &  \text{if the observation i is in cluster } k, \\
    \; 0 & \text{otherwise.}
  \end{cases}
\end{equation*}
The complete log-likelihood of the observations $(\bbX , \bbZ)$ is
defined by
\begin{equation*}
  L_{\theta}^{(c)} (\bbX_2, \bbZ_2) =  \sum_{i : X_i \in \bbX_2}
  \sum_{k=1}^K Z_{i, k} (  \ln p_k +  \ln \phi_{\mu_k,  \Sigma_k}(X_i) ). 
\end{equation*}
The algorithm consists of maximizing the expected value of the
log-likelihood with respect to the conditional distribution of $Z$
given $X$ under a current estimate of the parameters $\theta^{(r)}$:
\begin{equation*}
  Q(\theta | \theta^{(r)}) = \E_{Z | X ,\theta^{(r)}}  \Big[
  L_\theta^{(c)} (X,Z) \Big].
\end{equation*}
More precisely, the EM algorithm iterates the two following steps:
\begin{itemize}
\item Expectation step: compute the conditional probabilities
  \begin{equation} 
    \label{eq:probcond} 
    t_{i, k}(r) := \P(Z_{ik} = 1 | X,
    \theta^{(r)} ) = \frac{ p_k^{(r)} \phi_{\mu_k^{(r)},
        \Sigma_k^{(r)}}(X_i) } { \sum_{s=1}^K p_s^{(r)}
      \phi_{\mu_s^{(r)}, \Sigma_s^{(r)}}(X_i) },
  \end{equation}
  and then compute $Q(\theta |  \theta^{(r)})$.
\item Maximization step : find $\theta^{(r+1)} \in \argmax_{\theta \in
    \Theta(\eta)} Q(\theta |  \theta^{(r)})$.
\end{itemize}
After some iterations, $\theta^{(r)}$ is a good approximation of the
maximum likelihood estimator $\hat \theta(\eta)$.

At several steps of our method, we use the \emph{Maximum A Posteriori}
(MAP) rule to construct a clustering based on a GMM fit $\hat
\theta$. The MAP rule consists of attributing each observation $i$ to
the class which maximizes the conditional probability $\P(Z_{i, k} = 1
| X_i, \hat \theta )$. According to~\eqref{eq:probcond}, it reduces to
the choice of the class $\hat k (i) $ such that
\begin{equation*}
  \hat k(i) = \argmax_{k = 1, \ldots ,K}  \Big\{ \hat p_k \phi_{\hat
    \mu_k, \hat \Sigma_k}(X_i) \Big\}. 
\end{equation*}

\subsection{Metropolis-Hastings with a particular
  proposal}
\label{Sec:proposal}

To draw $\eta$ at random according to the generalized Bayes
posterior~\eqref{eq:generalized-bayesian-posterior}, we use the
Metropolis-Hastings (MH) algorithm
which
is of standard use for Bayesian statistics, see~\cite{MR2289769} and
for PAC-Bayesian algorithms, see~\cite{Catoni04}. The MH algorithm is
typically slow, so the careful choice of a proposal is often suitable
to obtain fast but statistically pertinent exploration.

In a supervised setting, such as regression, a natural idea for
variable selection is to add iteratively the variables that are the
most correlated with the residuals coming from a previous fit. This
idea leads to the so-called greedy algorithms, which are known to be
computationally efficient in a high-dimensional setting, see for
instance~\cite{MR2387964}.

In the unsupervised setting, however, no residuals are available, so 
one must come up with another idea. What is available though in
our case is the clustering coming from a previous GMM fit. So, when
exploring the variables space, it seems natural to give a stronger
importance to the variables that best explain the previous
clustering. This importance can be measured using the
between-variance.

\subsubsection*{The between-variance}

Let $\cC = \{\cC_1, \dots, \cC _K \} $ be a clustering into $K$ groups
of the observations $\bbX$. The between-variance of this clustering is
defined by
\begin{equation*}
  \varb (\cC) = \frac 1n \sum_{k=1}^K n_k \norm{G_k - G}^2 ,
\end{equation*}
where $n_k$ is the size and $G_k = (G_{k, 1}, \cdots, G_{k, d})^\top$
is the barycenter of cluster $\cC_k$, and where $G = (G_{1} , \cdots,
G_{d})^\top$ is the barycenter of the whole sample and $\norm{\cdot}$ is
the Euclidean norm on $\R^d$. The
between-variance of the variable $j$ is
\begin{equation*}
  \varb (\cC, j) = \frac 1 n \sum_{k=1} ^ K n_k  ( G_{k,j} - G_j )^2.
\end{equation*}

\subsubsection*{A particular proposal}
 
Assume that we are at step $u$ of the MH algorithm (see
Algorithm~\ref{algo:Hasting} below), the GMM estimated at this step has
configuration $\eta=(K,S)$ and is denoted by $g_u$. The clustering of $\bX$ deduced from this configuration with the MAP rule is denoted by $\mathcal
C(\eta)$. To
propose a new configuration $\tilde \eta=(\tilde K,\tilde S)$ for the step $u+1$, we use a transition kernel
$W$ defined as follows:
\begin{equation*}
  \forall \tilde \eta \in T, \quad W\left( \eta, \tilde \eta \right) = H(K, \tilde K) \: M_{ K, \tilde K}(S,\tilde S). 
\end{equation*}
The kernel $H$ determines how the number of clusters can change along the trajectory, it is defined by 
\begin{equation*}
  H(K,\cdot)  =  H^+(1, \cdot) \ind{K =1} \:  + \:   \frac{H^-(K, \cdot) + H^+(K, \cdot)}{2} \ind{ 1 < K <  K_{\max}}  \:  +  \: H^-(K_{\max},\cdot)
\ind{K = K_{\max}}
 \end{equation*}
with
  $H^+(K,K) = H^+(K,K+1) = 1/2$ and $H^-(K,K) = H^+(K,K-1) = 1/2$. \\
The kernel $ M_{K,\tilde K}$ determines how the number of active variables can change conditionally to the moves $K \rightarrow \tilde K$, it  is
defined by as follows:\\
$\bullet$  if $\tilde K \neq K$ :
$$
  M_{K,\tilde K}(S,\cdot) =  \ind{S = \tilde S}
$$ 
 $\bullet$  if $\tilde K = K$ : 
$$
  M_{K, K}(S,\cdot) = M_K^+(\emptyset,\cdot) \ind{|S| = 0} \: +   \: \frac{M_K^-(S, \cdot) +  M_K^+(S,\cdot)}{2} \ind{ 0 < |S| < d} \: +\: 
W_{2,K}^-(\{1,...,d\}, \cdot) \ind{|S| = d}
$$
with
\begin{eqnarray*}
  M_K^+(S, S \cup \{j\}) &=& \frac{ \varb(\cC(K,S), j)}{\sum_{j \notin S } \varb(\cC(K,S), j)}  \hskip 0.5cm \textrm{ if  } j \in S \textrm{ and $0$
otherwise} \\
  M_K^-(S, S \cup \{j\})   & = & \frac{\varb^{-1}(\cC(K,S), j)}{\sum_{j \in S} \varb^{-1}(\cC(K,S), j)} \hskip 0.5cm \textrm{ if  } j \notin S
\textrm{ and $0$ otherwise.}
\end{eqnarray*}

According to the transition kernel $W$, the set of variables can only change when the number of clusters is unchanged. When $\tilde K = K$, we decide
to add or remove one
variable with probability $1/2$. When adding a variable, we pick a
variable at random outside of $S$ according to the distribution
proportional to the between-variances $\varb(\cC_{u-1}, j)$. When
removing a variable, we choose a variable at random inside of $S$
according to the distribution proportional to the inverse
between-variances $\varb^{-1}(\cC_{u-1}, j)$. We observed empirically
that this proposal helps the MH exploration to focus quickly on
interesting variables.


Algorithm~\ref{algo:Hasting} below details our MH algorithm for one
split and one given temperature $\lambda$. 
This algorithm is similar to a stochastic greedy algorithm, which
decides to add or to remove a variable according to a criterion based
on an improvement of the likelihood, and constrained by the sparsity
of $K$ and $S$.

\begin{algorithm}[h]
  \caption{Metropolis-Hastings algorithm to draw $\eta \sim \hat
    \pi_\lambda$}
  \label{algo:Hasting}
  \begin{algorithmic}
    \REQUIRE $\bbX_1$, $\bbX_2$, $K_0$, $\lambda$
    \STATE Initialize $K \leftarrow K_0$%
    \smallskip
    \IF {$n \leq d$}%
    \STATE Take $S_0 \leftarrow \{ 1 \dots d\}$%
    \ELSE%
    \STATE Find a clustering $\cC$ with the $K$-means algorithm%
    \STATE Take $S_0$ as the set of $n$ variables $j$ maximizing
    $\varb(\cC, j )$%
    \ENDIF
    \smallskip
    \STATE Fit $g_0$ on $\bbX_2$ for the configuration $(K_0,S_0)$%
    \smallskip
    \STATE Find a clustering $\mathcal C_0$ of $\bbX$ with $g_0$%
    \smallskip
    \FOR{$u = 1$ to convergence}%
    \STATE Draw $\eta_{\new} = (K_{\new},S_{\new})$ distributed as
    $W(\eta_{u-1}, \cdot)$%
    \STATE Put $r \leftarrow \exp \Big[ \lambda \big( L_{\hat
      \theta(\eta_{\new})} (\bbX_1) - L_{\hat \theta(\eta_{u-1})}
    (\bbX_1) \big) \Big] \times \frac{ \pi(\eta_{\new})} {\pi
      (\eta_{u-1})}  \times \frac{ W ( \eta_{\new}, \eta_{u-1} ) }{ W( \eta_{u-1}, \eta_{\new})}$
    \STATE Draw $U$ distributed uniformly on $[0, 1]$%
    \IF{$U > r$}%
    \STATE Put $\eta_{u} \leftarrow \eta_{new}$%
    \STATE Fit $g_u$ on $\bbX_2$ for the configuration $\eta_{u} $%
    \STATE Find a clustering $\mathcal C _u$ of $\bbX$ with $g_u $%
    \ELSE \STATE Keep $\eta_{u} \leftarrow \eta_{u-1}$, $g_u
    \leftarrow g_{u-1}$ and $\mathcal C_u \leftarrow \mathcal C
    _{u-1}$%
  \ENDIF
\ENDFOR
    \smallskip
\RETURN $\eta_u$
\end{algorithmic}
\end{algorithm}

\subsubsection*{Preliminary pruning}

Algorithm~\ref{algo:Hasting} requires to compute an EM algorithm at
each step of the Markov chain. EM algorithms can be time consuming, in
particular for GMM based on a large family of active variables. To
quickly focus on a reasonably large set of active variables, we speed
up the procedure using a rough ``preliminary pruning''. We replace
$W_2^-$ by a proposal that allows to remove several variable in one
step, until the number of active variables is smaller than a ``target
number'' fixed by the user, we used $d/3$ in our
experiments. Moreover, the number of variables to be removed is drawn
at random, and cannot be larger than half of the remaining active
variables. Once the target number of active variables is reached, we
use the proposal $W_2^-$ described above. Note that this rough
preliminary pruning may remove sometimes true variables, but,
fortunately, they can be recovered later.

\subsubsection*{An illustrative example}


We illustrate the complete procedure on a simulated GMM based on two
components in dimension~100. Only the 15 first variables are active
and the others are i.i.d. with distribution $N (0, 1)$. The mean of
the first 15 variables for each component are respectively $\mu_1 =1$
for cluster~1 and $\mu_2 =0$ for cluster~2. Both clusters have 100
observations and the whole sample has been centered and standardized.
We use only here GMM with shapes $\shape_{LB}(\eta)$ defined
by~\eqref{eq:shape-lb}.

Figure \ref{fig:IllustrMethodeTraj.pdf} shows nine trajectories of
active variables, each trajectory being relative to a temperature taken
from the grid $\Lambda = \{ 1, 2, 5, 10, 20, 30, 50, 75, 100\}$. All the trajectories
 are based on the same split. The chains have been initialized on
the configuration $K_0 = 2$ and $S_0 = \{1, \dots, d\}$.  We used the
preliminary pruning here to quickly reduce the number of
variables. After a few dozen of steps, the number of active variables
has been dramatically reduced. Note that we use chains of length 300
only to reach convergence. There is no need to wait any longer since
the chains are already stabilized at this stage around the correct
configuration.

\begin{figure}[h]
\begin{center}
  \includegraphics[width=1\textwidth,
  angle=0]{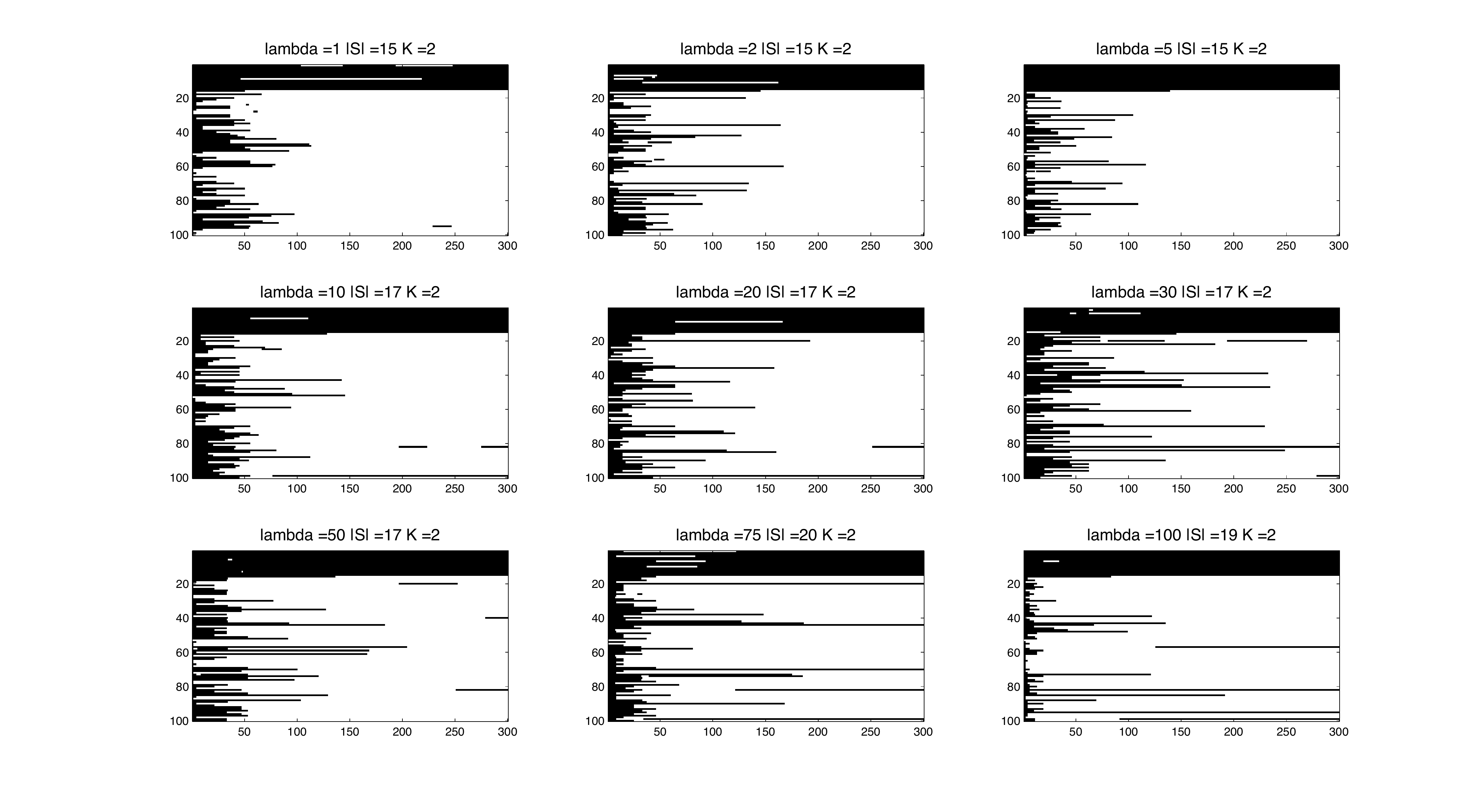}%
\end{center}
\vspace{-1cm}
\caption{MH trajectories with temperatures in $\Lambda$ and for a same
  split. These graphs show the actives variables in black. The number
  of components at the end of the chain is given in the title of each
  graph.}%
\label{fig:IllustrMethodeTraj.pdf}%
\end{figure}

\subsection{Post-treatment of the trajectories}
\label{subsec:PostTrait}

The MH algorithm presented in the previous section is proceeded for
several splits and several temperatures. Leaving aside for the moment the  temperature choice issue, we then have $B$ available
 GMM estimators of the density. This generates a large amount of information that can be used to cluster the data $\bbX$. One first idea is to
aggregate all these estimators to provide one final estimator of the density. Nevertheless, this estimator can not be easily used to produce a
relevant clustering of the data since this last is a GMM with all the composants of each GMM. We thus
need to aggregate the information provided by the splits in another
way. We propose here two  alternative methods : the first one consists in selecting one final GMM and the second method consists in aggregating the
clustering provided by the $B$ clusterings.

\subsubsection*{Selection of one configuration for each chain}  

We associate one configuration $\eta(b,\lambda)$ to each split $b$ and each
temperature $\lambda$ by choosing the most visited configuration at the end of the chain. This
choice is indeed reasonable because the chain state is very stable
when $u$ is large enough as shown in Figure~\ref{fig:IllustrMethodeTraj.pdf} (in practice we consider the
last $100$ visited configurations). Note that most of the time this configuration is also the last configuration visited by the chain.

\subsubsection*{Temperature choice}

We choose a temperature $\lambda$ for each split $b$ in the following
way. We fit a GMM for the configuration $\eta(b,\lambda)$ using
$\bbX$. Let us denote by $\hat \theta(b,\lambda)$ the parameters of
this fit. For a given split $b$, 
we choose a temperature according to
BIC or AIC criteria:
\begin{equation*}
  \lambda_{\aic}(b) = \argmin_{\lambda \in \Lambda}  - 2  L_{\hat
    \theta(b,\lambda)}(\bbX) + |\eta(b,\lambda)|
\end{equation*}
and
\begin{equation*}
  \lambda_{\bic}(b) = \argmin_{\lambda \in \Lambda}  - 2  L_{\hat
    \theta(b,\lambda)}(\bbX) +  \ln n |\eta(b,\lambda)|,
\end{equation*}
where $| \eta |$ is the number of free parameters of the GMM
associated to the configuration $\eta$.
This gives for each split $b$ one configuration $\eta(b)$. Figure
\ref{fig:IllustrMethodeSplits.pdf} shows the supports chosen for each
split using AIC and BIC criteria.
\begin{figure}[h]
  \begin{center}
    \includegraphics[width= 1\textwidth,
    angle=0]{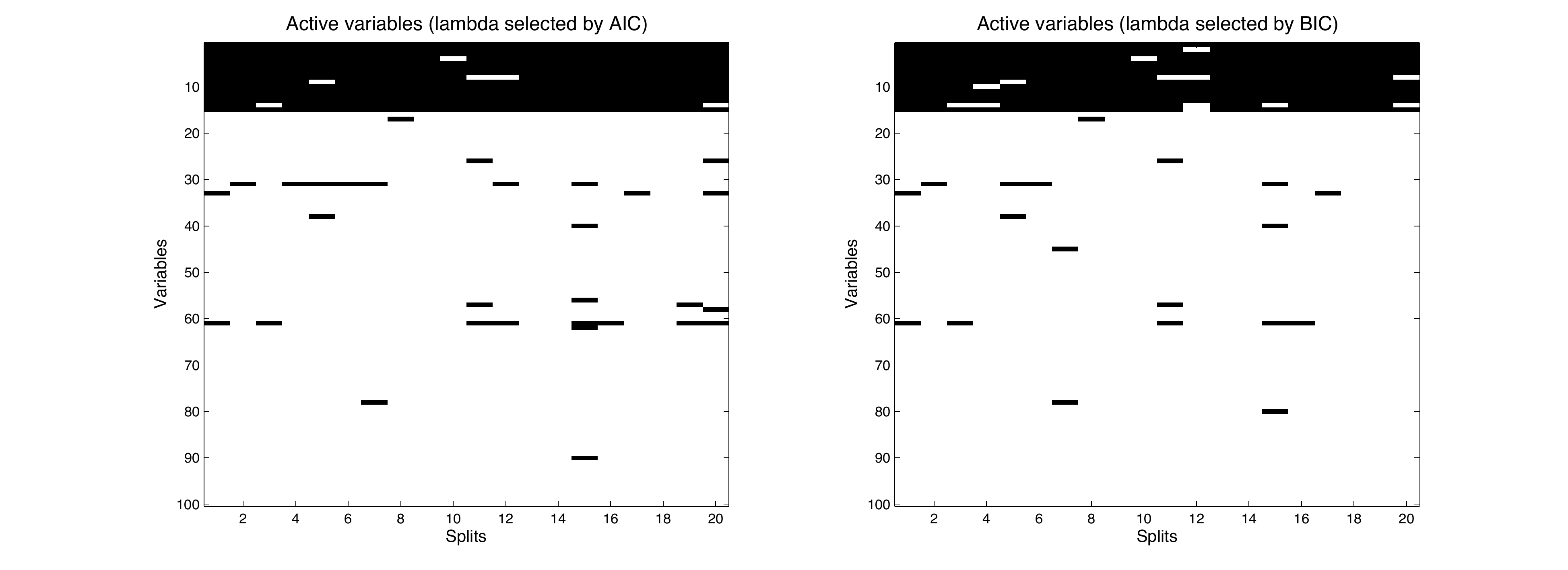}
    \vspace{-1.2cm}
  \end{center}
  \caption{Chosen supports for each split using AIC and BIC criteria.}%
  \label{fig:IllustrMethodeSplits.pdf}%
\end{figure}

\subsubsection*{Final configuration selection}

The information brought by the family of splits can be used to measure
the ``importance'' of each variable by considering the proportion of
splits for which this variable is active, see the
Figure~\ref{fig:IllustrMethodesKImport.pdf} (right). Note that this
measure could also be used to produce a ranking of the variables. One
final configuration $\hat \eta$ is finally chosen using a majority
vote: for each variable, we vote over the $B$ splits to decide if this
variable is active or not. We choose $\hat K$ as the integer the
closest to $ \frac 1B \sum_{b = 1 \ldots B} K(b)$.

The method has been applied on the illustrative example using $20$
splits. Figure~\ref{fig:IllustrMethodesKImport.pdf} (left) shows the
number of components $K$ chosen over the 20 splits: $K=2$ is
majoritarian if the temperature is chosen whether by AIC or BIC. For
this experiment, the exact family of active variable is also correctly
recovered by the voting method with a temperature selected by AIC or
BIC.

\begin{figure}[h]
  \begin{center}
      \includegraphics[width= 0.49    \textwidth]{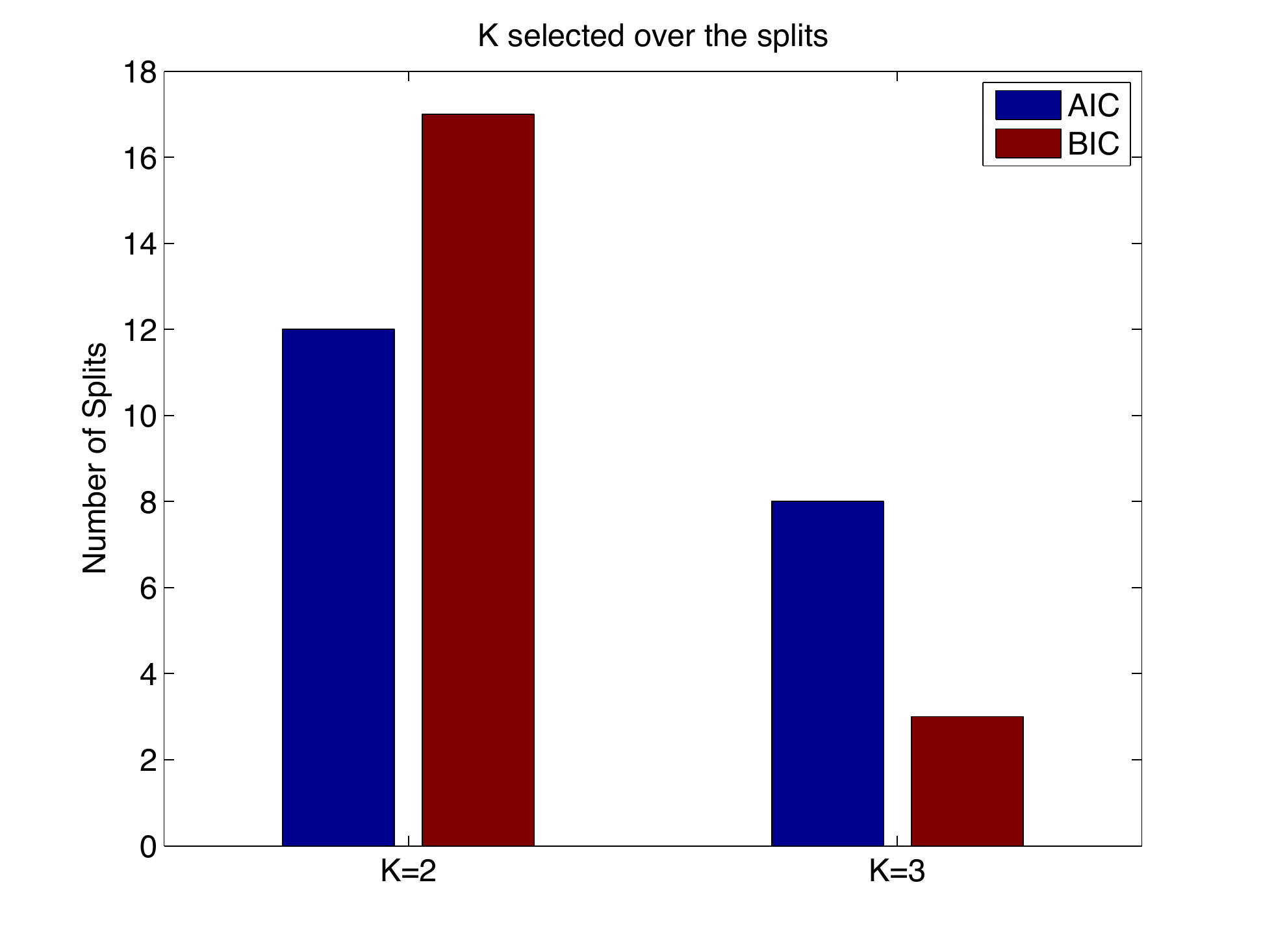}
    \includegraphics[width= 0.49   \textwidth]{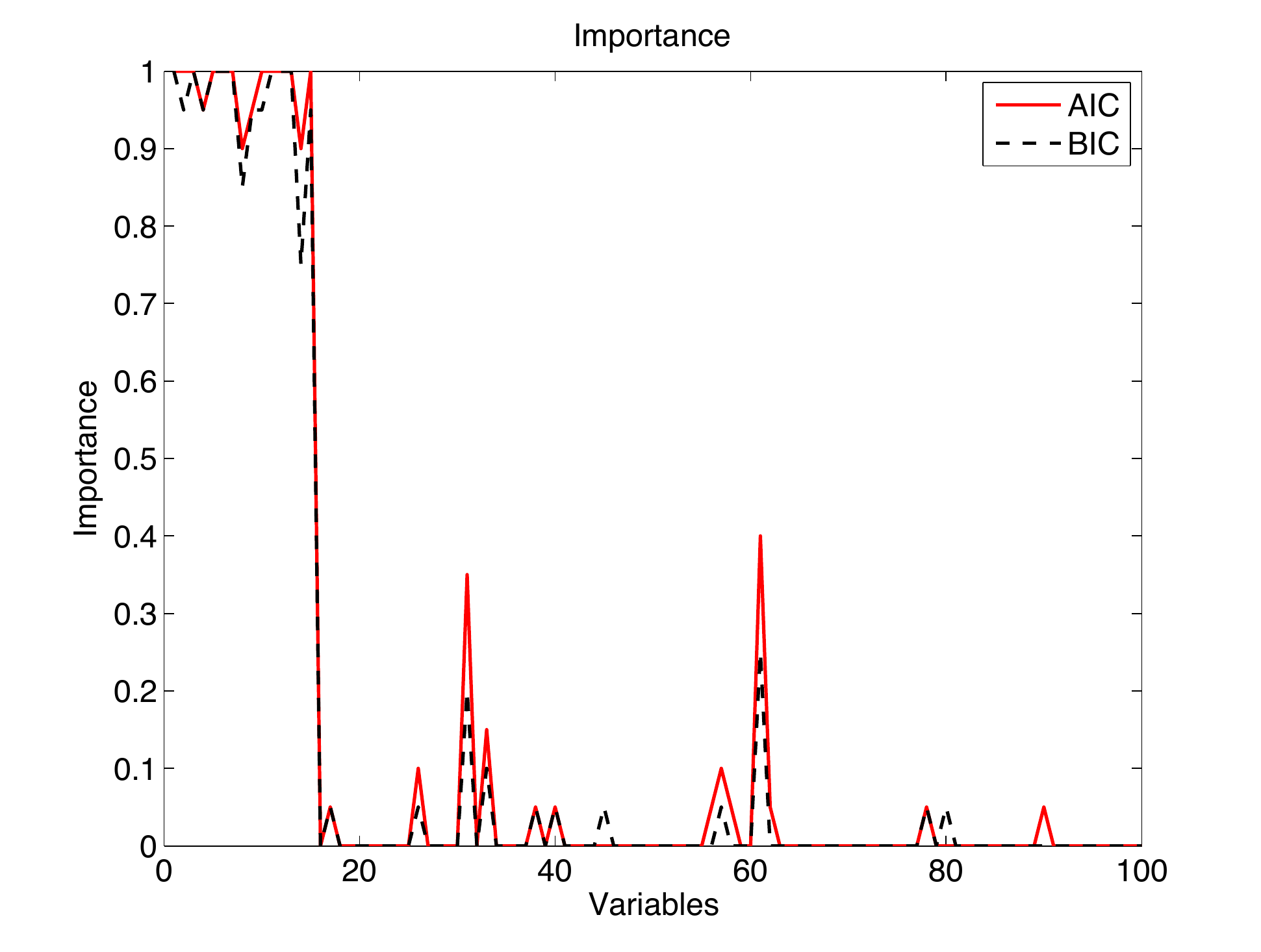}
  \end{center}
  \vspace{-0.9cm}
  \caption{Left: number of components chosen over the 20 splits using
    AIC or BIC to set the temperature for each split. Right: variable
    importance measured over the 20 splits.}%
  \label{fig:IllustrMethodesKImport.pdf}%
\end{figure}

\subsubsection*{Final Clustering}

One can think of two strategies to define a final clustering.
\begin{itemize}
\item \emph{Direct clustering from $\hat \eta$}: we fit a GMM model on
  $\bbX$ for the configuration $\hat \eta$ chosen by the method
  detailed above. We then propose a clustering for $\bbX$ using the
  MAP rule.
\item \emph{Aggregated clustering}: each split provides a
  configuration $\eta(b)$ and an associated clustering thanks to the
  MAP rule. Many methods exists to produce an aggregated clustering using these $B$ clusterings available, we propose two versions :
  \begin{itemize}
  \item {\it Aggregated clustering by CAH}: define $A$ as the
    similarity matrix with entries $a_{i,j}$ equal to the number of
    times $i$ and $j$ are in the same cluster across the splits. Then,
    a hierarchical clustering method gives us a final clustering by
    using for instance $\exp(-A)$ as a dissimilarity matrix.
  \end{itemize}
\end{itemize}
To assess a clustering method, we use the Adjusted Random Index (ARI)
from~\cite{HubertArabie85}, which is an established standard to
measure the correspondences between a given clustering and the true
one. The ARI's of our methods on the illustrative example are reported
in Table~\ref{tab:expIllus}. In this case, a direct clustering deduced
from the configuration $\hat \eta$ gives the best results. Indeed, we
observed on many examples that, quite surprisingly, the aggregated
clustering generally does not provide the best ARI's compared to
direct clustering. This phenomenon certainly deserves further
investigations, to be considered in another work.
\begin{table}
  \centering
  \begin{tabular}{c|c}
    Method  & ARI \\ \hline
    Direct clustering (temperature set by AIC) & 0.9020  \\
    Direct clustering (temperature set by BIC) & 0.9020  \\
    Aggregated clustering  (temperature set by AIC) & 0.8090 \\
    Aggregated clustering (temperature set by BIC) & 0.7557 \\
  \end{tabular}
  \caption{Adjusted Random Indexes (ARI) for the four possible
    clustering methods
    applied to the illustrative example.} 
  \label{tab:expIllus}
\end{table}

\section{Numerical experiments}
\label{sec:numerical-experiments}

We compare variable selection and clustering results for the Lasso
method for GMM from \cite{pan2007penalized} (denoted Lasso-GMM) and
our method (called MH-GMM). On each simulation, the observations are
centered and standardized before being used by both methods. We use
the BIC criterion to tune the temperature for our method and the
smoothing parameter for Lasso-GMM.  In the following experiments, we
only deal with diagonal and shared covariance matrices (the
shape~\eqref{eq:shape-lb}), which is also the setting considered
in~\cite{pan2007penalized}.




\subsubsection*{Experiment~1}


We simulate $100$ replications of a GMM based on three components in
dimension 100. All the variables are independent, only the 20 first
variables are active and the others are i.i.d.  $N (0,1)$. Let $a =
(1, 0.95, 0.9, 0.85, \dots, 0.1, 0.05)$, the distribution of the
vector of the 20 active variables of the first component is $N _{20}
(a, I_{20})$. The distribution of the vector of the second component
is $N_{20} (0_{20}, I_{20})$ and the third is $N_{20} (-a,
I_{20})$. There are $200$ observations in the two first clusters and
$400$ in the last one.

On this example, the discriminant power of the clustering variables
decreases with respect to the variable index: the three sub-populations
of the mixture are progressively gathered together into a unique
Gaussian distribution after the $20$th variable, as illustrated in the
left-hand side of Figure~\ref{fig:Triangle}. Note that the means of
the second component are close to zero.

\subsubsection*{Experiment~2}


We simulate $100$ replications of a GMM based on four components in
dimension 100. These four clusters are based on 15 actives variables
(from 1 to 15), the variances of the active variables are equal to one
for the four components. The other variables are i.i.d. $N(0, 1)$. All
the variables are independent and there are $100$ observations for
components 1 and~4 and $200$ for components~2 and~3. The means of the
four components are given by
\begin{align*}
  \mu_{1} &= (2, \ldots, 2, 0, \ldots, 2), \quad \quad \quad \quad
  \mu_{2} = (0.3, \ldots, 0.3, 0, \ldots, 0), \\
  \mu_{3} &= (-0.3, \ldots, -0.3, 0, \ldots, 0), \quad \mu_{4} = (-2,
  \ldots, -2, 0, \ldots, 0),
\end{align*}
with $15$ non-zero coordinates for each component. For this
experiment, selecting $K$ is more difficult than in the previous
experiment. The mean of each cluster is illustrated in the right-hand
side of Figure~\ref{fig:Triangle}.
\begin{figure}
  \begin{center}
    \includegraphics[width= 0.5\textwidth,
    angle=0]{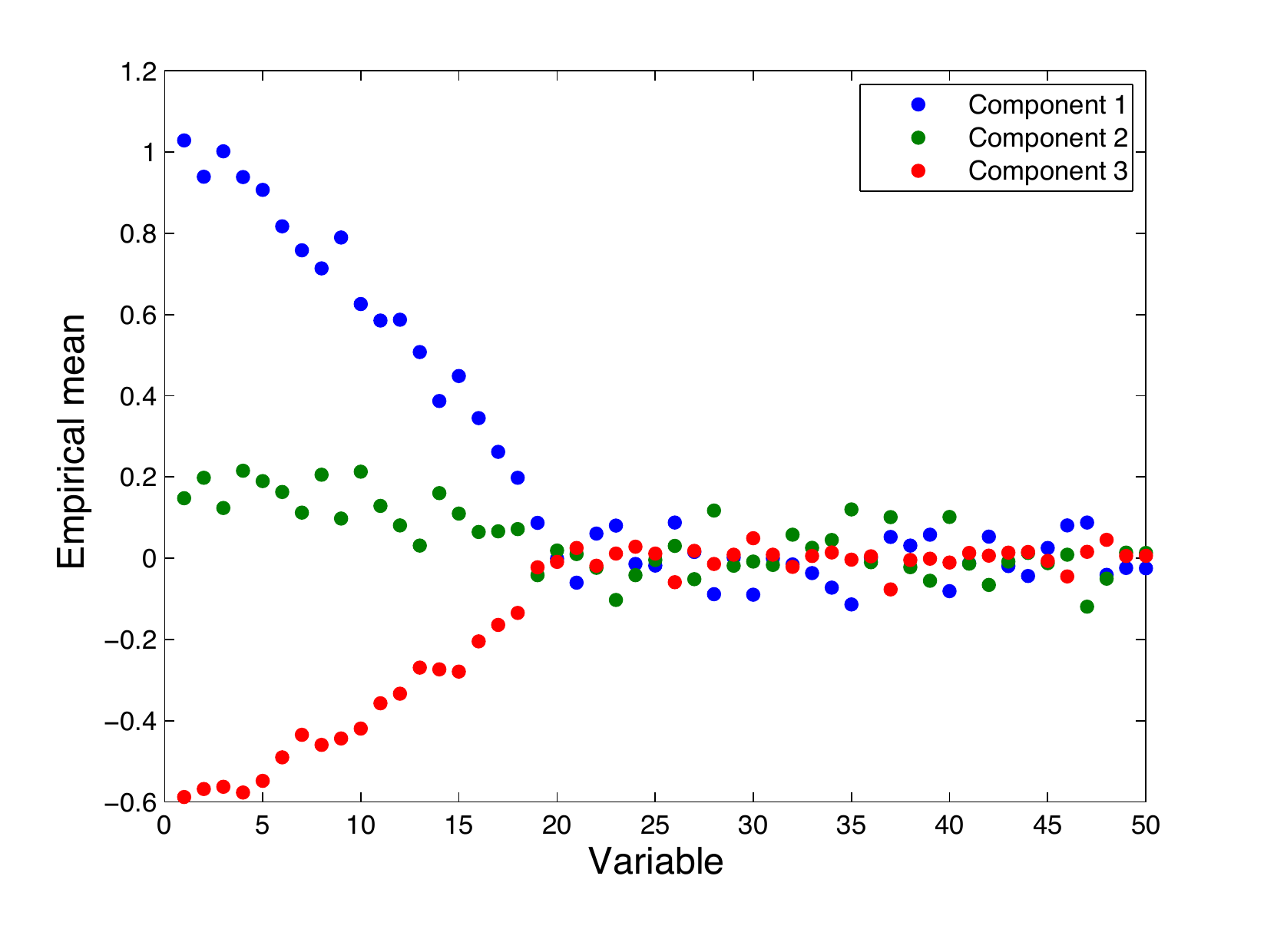}%
    \includegraphics[width=0.5\textwidth,
    angle=0]{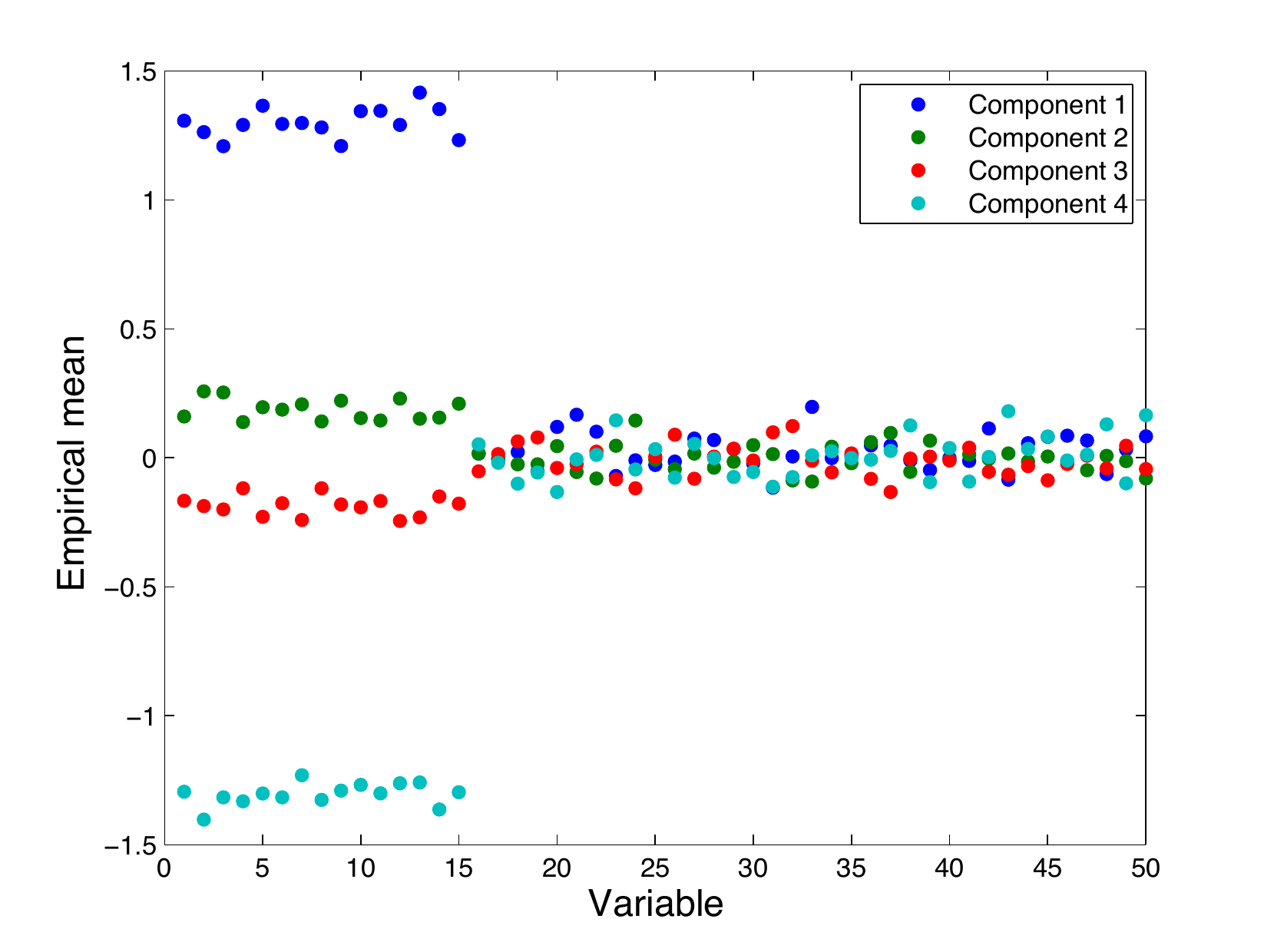}
  \end{center}
  \vspace{-0.9cm}
  \caption{Empirical means of the variables in each component for
    Experiment~1 (left) and Experiment~2 (right).}
  \label{fig:Triangle}%
\end{figure}


\subsubsection*{Results}

We compare variable selection results and clustering results using the
ARI for Lasso-GMM and MH-GMM on Experiments~1 and~2. A variable
selected by one method is called active. Concerning Lasso-GMM, a
variable is active if at least one $\mu_{j, k}$, $k=1, \ldots, K$ is
non-zero while for MH-GMM all the coordinates associated to an active variable
variables are non-zero. The terms true and false refers to the true
configuration. The results are summarized in Table~\ref{tab:variables-K} and
Figure~\ref{fig:ARI-exp1-2}.

On Experiment~1, Lasso-GMM finds the correct active variables but most
of the times it does not activate the coordinates corresponding to the
second component. This leads to a non-optimal estimation of the GMM
even if it has the correct configuration. As a consequence, a large
proportion of observations are misclassified. The MH-GMM method does
not shrink coordinates, so it shows better ARI rates.

On Experiment~2, Lasso-GMM fails to select $K = 4$ correctly 72 times
out of 100 replications while MH-GMM chooses $K = 4$ every
time. Indeed, Lasso-GMM hardly separates Cluster~2 from
Cluster~3. Consequently, MH-GMM has a much better ARI clustering rate
than Lasso-GMM, see the right-hand side of
Figure~\ref{fig:ARI-exp1-2}.


\begin{table}
  \centering
  \begin{tabular}{r|cccc|cc}
    \textbf{Experiment~1} & \multicolumn{4}{c}{Variable selection} &
    \multicolumn{2}{c}{Choice of $K$} \\ \hline%
    & true active & false active & true
    non-active & false non-active & $K = 3$ & $k = 4$ \\%
    \emph{True model} & 20 & 0 & 80 & 0 & 100 & 0 \\%
    Lasso-GMM & 18.56 & 2.17 & 77 & 1.44 & 99 & 1 \\%
    MH-GMM & 16.82 & 0 & 80 & 3.18 & 100 & 0 \\%
    \textbf{Experiment~2} & \multicolumn{4}{c}{Variable selection}
    & \multicolumn{2}{c}{Choice of $K$} \\ \hline%
    & true active & false active & true
    non-active & false non-active & $K = 3$ & $k = 4$ \\%
    \emph{True model} & 15 & 0 & 85 & 0 & 100 & 0 \\%
    Lasso-GMM & 15.03 & 0.81 & 84.19 & 0 & 72  & 28 \\%
    MH-GMM & 15 & 0 & 85 & 0 & 0 & 100 \\ 
  \end{tabular}
  \caption{Variables and $K$ selected by Lasso-GMM and MH-GMM for
    Experiments~1 and~2.}
  \label{tab:variables-K}
\end{table}

\begin{figure}
  \begin{center}
    \includegraphics[width=0.5\textwidth]{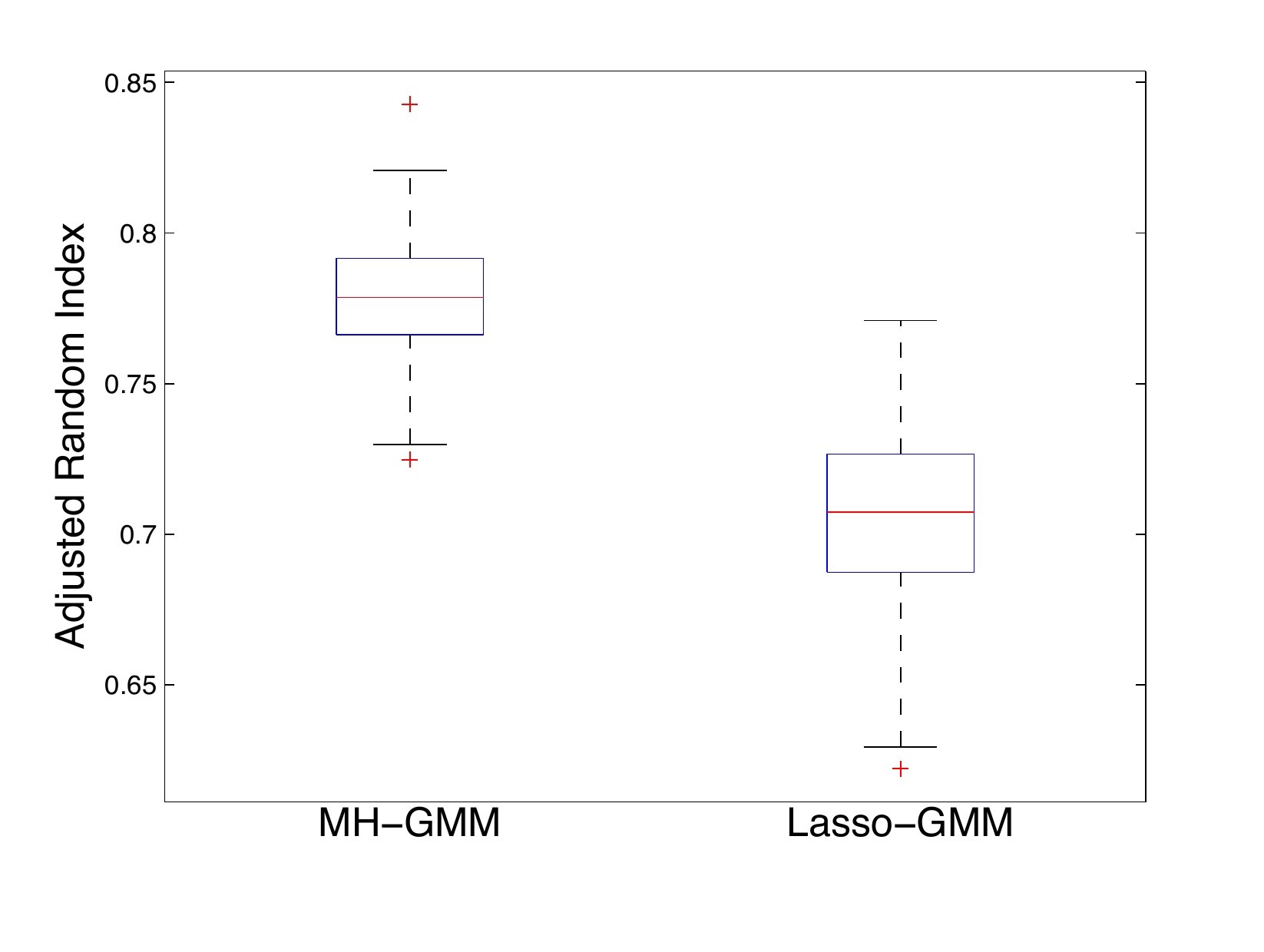}%
    \includegraphics[width=0.5\textwidth]{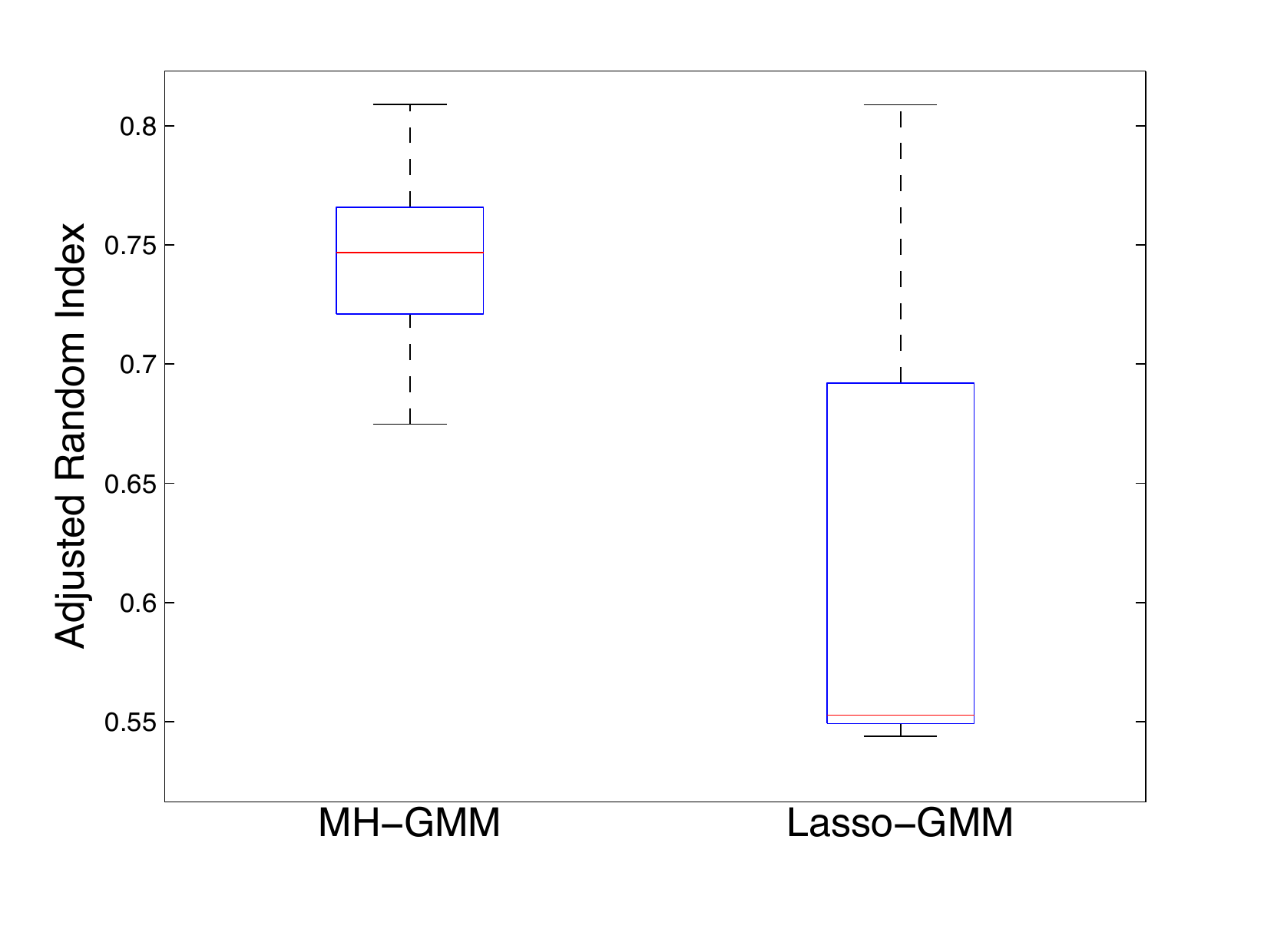}
  \end{center}
  \vspace{-1cm}
  \caption{ARIs for Lasso-GMM and MH-GMM for Experiments~1 and~2}%
  \label{fig:ARI-exp1-2}%
\end{figure}

\section{Proof of Theorem~\ref{thm:main-oracle-inequality}}
\label{sec:proof-thm1}
In this section, we assume that a shape has been fixed. 

\subsection{Some preliminary tools from information theory}

A measure of statistical complexity of a randomized estimator $\hat
\pi$ with respect to a prior $\pi$ is given by the Kullback-Leibler
divergence, defined by
\begin{equation*}
  \kl(\hat \pi, \pi) = \int_{\Upsilon} \ln \frac{d\hat \pi}{d
    \pi}(\eta) \hat \pi(d\eta),
\end{equation*}
assuming that it exists. The statistical risk, or generalization error
of a procedure $\hat \pi$ associated to a loss function
$L_\eta(\cdot)$ on $\Ups \times \cX$ is given by
\begin{equation*}
  \E_{\eta \sim \hat \pi} \E_{X} L_\eta(X) = \int \int L_\eta(x) \hat
  \pi(d\eta) P_X(dx).
\end{equation*}
Let $\pi'$ and $\pi$ be probability measures on $\Ups$, and let $g :
\Ups \rightarrow \R$ be any measurable function such that $\E_{\eta
  \sim \pi} \exp(g(\eta)) < +\infty$. An easy computation gives, for
$\pi_g(d\eta) = \frac{\exp(g(\eta))}{\E_{\eta \sim \pi} \exp(g(\eta))}
\pi(d \eta)$, that
\begin{equation}
  \label{eq:convex-duality1}
  \kl(\pi', \pi_g) = \kl(\pi', \pi) + \ln
  \E_{\eta \sim \pi} \exp(g(\eta)) - \E_{\eta \sim \pi'} g(\eta).
\end{equation}
Since $\kl(\pi', \pi_g) \geq 0$ (using Jensen's inequality)
\eqref{eq:convex-duality1} entails
\begin{equation}
  \label{eq:convex-duality2}
  \E_{\eta \sim \pi'} g(\eta) \leq \kl(\pi', \pi) + \ln
  \E_{\eta \sim \pi} \exp(g(\eta)).
\end{equation}
Inequality~\eqref{eq:convex-duality2} is a well-known convex duality
inequality, see for instance~\cite{Catoni07}, that can be used as
follows. We denote $\bbX_1 = (X_1, \ldots, X_n)$ and let
$L_\eta(\bbX_1)$ be some loss
function. Using~\eqref{eq:convex-duality2} with a randomized estimator $\hat
\pi$ and with the application
\begin{equation*}
  g(\eta) = - L_\eta(\bbX_1) - \ln \E_{\bbX_1} e^{-L_\eta(\bbX_1) }
\end{equation*}
leads to
\begin{equation*}
  \E_{\bbX_1} \exp \bigg( \E_{\eta \sim \hat \pi} \Big[ - L_\eta(\bbX_1) - 
  \ln \E_{\bbX_1} e^{-L_\eta(\bbX_1) } - \kl(\hat \pi, \pi) \Big] \bigg) \leq 1,
\end{equation*}
where we apply $x \mapsto e^x$, take the expectation $\E_{\bbX_1}$ on
both sides of~\eqref{eq:convex-duality2}, and use Fubini's
Theorem. Now, if $L_\eta(\bbX_1) = \sum_{i=1}^n \ell_\eta(X_i)$, we
have using Jensen's inequality on the left-hand side, and rearranging
the terms, that
\begin{equation}
  \label{eq:localized-posterior-bound}
  - \E_{\bbX_1} \E_{\eta \sim \hat \pi} \ln \E_{X_1} e^{-
    \ell_\eta(X_1) } \leq \frac{\E_{\bbX_1} \big[ \E_{\eta \sim \hat
      \pi} L_\eta(\bbX_1) + \kl(\hat \pi, \pi) \big]}{n}.
\end{equation}
This leads to the so-called \emph{information complexity
  minimization}, see~\cite{Zhang06a}: the statistical procedure
obtained by minimizing
\begin{equation}
  \label{eq:characterization-posterior}
  \E_{\eta \sim \hat \pi} L_\eta(\bbX_1) +
  \kl(\hat \pi, \pi)
\end{equation}
over every probability measures $\hat \pi$ on $\Ups$ is explicitly
given by
\begin{equation*}
  \hat \pi_{\text{sol}}(d\eta) = \frac{\exp(-L_\eta(\bbX_1))}{\E_{\eta \sim
      \pi} \exp(-L_\eta(\bbX_1))} \pi(d \eta),
\end{equation*}
since~\eqref{eq:convex-duality1} entails
\begin{equation*}
  \kl(\hat \pi, \hat \pi_{\text{sol}}) = \E_{\eta \sim \hat \pi}
  L_\eta(\bbX_1) + \kl(\hat \pi, \pi) + \ln \E_{\eta \sim \pi}
  \exp( - L_\eta(\bbX_1)).
\end{equation*}
Note that if $L_\eta(\bbX_1) = \sum_{i=1}^n \ell_\eta(X_i)$ with
$\ell_\eta(x) =  - \lambda  \ln f_{\hat \theta(\eta)}(x)$ with $\lambda
> 0$, the solution is given by the Gibbs
posterior~\eqref{eq:gibbs-posterior}, which is in this case the
generalized Bayesian posterior
distribution~\eqref{eq:generalized-bayesian-posterior}.

For any $\rho \in (0, 1)$ we can define the $\rho$-divergence between
probability measures $P, Q$ as
\begin{equation*}
  \cD_\rho(P, Q) = \frac{1}{\rho(1-\rho)} \E_P\Big[ 1 -
  \Big( \frac{dQ}{dP} \Big)^\rho \Big],
\end{equation*}
the case $\rho = 1/2$ giving $\cD_\rho(P, Q) = 4 \cH(P, Q)$, where
$\cH(P, Q) = \E_P(1 - \sqrt{dQ / dP})$ is the Hellinger distance. Note
that
\begin{equation}
  \label{eq:hellinger-rho-div}
  \frac 12 \cH^2(P, Q) \leq \max(\rho, 1-\rho) \cD_\rho(P, Q)
\end{equation}
for any $\rho \in (0, 1)$, see~\cite{Zhang06b}.

\subsection{Proof of Theorem~\ref{thm:main-oracle-inequality}}

Since $1 - x \leq -\ln x$ for any $x \geq 0$, we have
\begin{equation*}
  \lambda (1 - \lambda)   \cD_\lambda(f^*, f_{\hat \theta(\eta)}) \leq -  \ln \E_{X_1} \exp ( -\ell_\eta(X_1) ),
\end{equation*}
where we consider $\ell_\eta(x) = \lambda \ln \frac{f^*(X)}{f_{\hat
    \theta(\eta)}(X)}$. So, using~\eqref{eq:localized-posterior-bound}
and the definition of $\hat \pi_\lambda$, we have
\begin{align*}
  \lambda (1 - \lambda) \E_{\bbX_1} \E_{\eta \sim \hat \pi_\lambda}
  \cD_\lambda(f^*, f_{\hat \theta(\eta)}) &\leq \E_{\bbX_1} \Big[
  \frac{\lambda}{n} \E_{\eta \sim \hat \pi_\lambda} \sum_{i=1}^n \ln
  \frac{f^*(X_i)}{f_{\hat \theta(\eta)}(X_i)} + \frac 1n \kl(\hat
  \pi_\lambda, \pi) \Big] \\
  &= \E_{\bbX_1} \Big[ \inf_{\hat \pi} \Big( \frac{\lambda}{n} \E_{\eta
    \sim \hat \pi} \sum_{i=1}^n \ln
  \frac{f^*(X_i)}{f_{\hat \theta(\eta)}(X_i)} + \frac 1n \kl(\hat
  \pi, \pi) \Big) \Big] \\
  &\leq \inf_{\hat \pi} \Big( \lambda \E_{\eta \sim \hat \pi} \kl(f^*,
  f_{\hat \theta(\eta)}) + \frac 1n \kl(\hat \pi, \pi) \Big),
\end{align*}
where the infimum is taken among any probability measure on $\Ups$.
Using~\eqref{eq:hellinger-rho-div}, this leads to
\begin{align*}
  \frac 12 \E_{\bbX_1} \E_{\eta \sim \hat \pi_\lambda} \cH(f^*,
  f_{\hat \theta(\eta)}) &\leq \max(\lambda, 1-\lambda) \E_{\bbX_1} \E_{\eta
    \sim \hat
    \pi_\lambda} \cD_\lambda(f^*, f_{\hat \theta(\eta)}) \\
  &\leq \frac{\max(\lambda, 1 - \lambda)}{\lambda(1 - \lambda)}
  \inf_{\hat \pi} \Big( \lambda \E_{\eta \sim \hat \pi} \kl(f^*,
  f_{\hat \theta(\eta)}) + \frac 1n \kl(\hat \pi, \pi) \Big).
\end{align*}
By considering only the subset of Dirac distributions over $\Ups$, we
obtain
\begin{align*}
  \E_{\bbX_1} \E_{\eta \sim \hat \pi_\lambda} \cH(f^*, f_{\hat
    \theta(\eta)}) &\leq c_\lambda \inf_{\substack{K \in \N^* \\ S
      \subset \{ 1, \ldots, d \}}} \bigg( \lambda \kl(f^*, f_{\hat
    \theta(K, S)}) + \frac{1}{n}  \ln \Big( \frac{1}{\pi(K, S)} \Big)
  \bigg),
\end{align*}
where $c_\lambda$ is defined in the statement of
Theorem~\ref{thm:main-oracle-inequality}. Since
\begin{equation*}
   \ln \Big( \frac{1}{\pi(K, S)}
  \Big)  =  \ln \Big( \frac{1}{\pi_{\text{clust}}(K)}
  \Big) +  \ln \Big( \frac{1}{\pi_{\text{supp}}(S)}
  \Big) \leq  \ln K! + 1 + 2 |S|  \ln \Big( \frac{e d}{|S|}
  \Big),
\end{equation*}
it gives that
 \begin{equation*}
    \E_{\bbX_1}  \E_{\eta \sim \hat \pi_\lambda} \cH^2(f^*,
     f_{\hat \theta(\eta)}) \leq c_\lambda \inf_{\substack{K \in \N^* \\ S
        \subset \{ 1, \ldots, d \}}} \bigg\{ \lambda \kl(f^*,
    f_{\hat \theta(K, S)}) + \frac{ \ln K! + 1 + 2 |S|  \ln ( e
      d / |S|) )}{n} \bigg\} .
  \end{equation*}
The other inequalities given in Theorem~\ref{thm:main-oracle-inequality} are straightforward  using the convexity properties of
the Hellinger distance (see for instance Lemma 7.25 in \cite{massart03}).

\subsection{Proof of Theorem~\ref{theo:BorneSupAvecMLE}}

We start with an elementary lemma for bounding sup-norm of density ratios.
\begin{Lemma}  \label{lem:normsup}
Let $f^\star$ be density in $\R^d$ such that $\| f^\star \| _\infty <  \infty $ and such that the
support of $f ^\star$ is included in $B(0,\bar \mu)$. Then, for any GMM shape, for any $ \eta \in \Ups $ and any $\theta \in \Theta_r $:
$$ 1 \leq  \left\| \frac {f ^\star} {f_{\theta}}\right\|_{\infty} \leq \|f^\star\|_\infty   L ^ +  \exp \left(  \frac {2 {\bar \mu}^2} {\overline{\sigma} ^2 }  \right) $$ 
\end{Lemma}
\begin{proof}
Let $\eta = (K,S) \in \Ups$ and $\theta \in \Theta_c $.  We have $\left\|\frac {f^\star}{ f_\theta} \right\|_{\infty} \geq 1 $ because $f^\star \ll f_\theta  $ since $f ^\star$ has a bounded support. For the other inequality, first  note that
\begin{equation} \label{eq:rapportborne}
\left\|\frac {f^\star}{ f_\theta} \right\|_{\infty} \leq  \frac{\|f\|_\infty}{\inf_{x \in B(0,\mu)} f _\theta}
\end{equation}
According to the constraints on the determinant of the covariance matrices, for any $x \in B(0, \mu)$:
\begin{eqnarray*}
f_\theta(x) & \geq & \frac 1 {L ^ +}  \sum _{k=1}^K p_k \exp \left[- \frac 1 2 (x - \mu_k)' \Sigma _ k  ^{-1}  (x- \mu_k)\right]  \\
& \geq & \frac 1 {L ^ +}   \exp \left[-\frac 1 2  \sum _{k=1}^K p_k  \|x - \mu_k\| ^2 \max \left(\mbox{sp}( \Sigma^{-1})\right) \right]  \\
& \geq & \frac 1 {L ^ +}   \exp \left[-   \frac { 2  {\bar \mu}^2} {\overline{\sigma} ^2 }  \right] .
\end{eqnarray*}
and the Lemma is proved using (\ref{eq:rapportborne}).
\end{proof}

It is well known that rates of convergence of ML-estimators can be stated by computing bracketing entropies of the statistical models involved, (see \cite{WongShen95,vandeGeer2000} among others). Let $\mathcal F$ be a set of densities with respect of the Lebesgue measure. An $\varepsilon$-bracketing for $\mathcal F$  with respect to $\cH$ is a set of integrable function pairs $(l_1,m_1),\ldots,(l_N,m_N)$ such that for each $f\in \mathcal F$, there exists $j\in\{1,\ldots,N\}$ such that $l_j\leq f\leq m_j$ and $\cH(l_j,m_j)\leq \varepsilon$. The bracketing number $\mbox{N}_{[.]}(\varepsilon,\mathcal F,\cH)$ is the smallest number of $\varepsilon$-brackets necessary to cover $\mathcal F$ and the bracketing entropy is defined by 
$\mbox{H}_{[.]}(\varepsilon,\mathcal F,\cH)= \ln \left\{\mbox{N}_{[.]}(\varepsilon,S,\cH)\right\}$. 

Rates of convergences of MLE in GMM was first studied by  \cite{Genovese:00} and \cite{Ghosal:01}, following a method introduced by \cite{WongShen95}.
Here we use the following result which can be, easily rewritten from the proof of Theorem 7.11 in \cite{massart03}. This result gives an
exponential deviation bound for the Hellinger risk of a maximum likelihood estimator. Note that for our problem we do not need an uniform control of
the risk of the estimators over the  model collection. 

Assume that there exists a nondecreasing function : $\Psi$ such that $x\rightarrow\Psi (x)/x$ is nonincreasing on $]0,+\infty[$ and such that for any
$\xi \in \mathbb{R}_{+}$ and any $u \in \mathcal F$:
\begin{equation} \label{eq:intDudley}
\int_{0}^{\xi} \sqrt{\mathcal{H}_{[.]}(x,\mathcal F}(g,\xi),\cH)  \, dx \leq  \Psi(\xi)
\end{equation}
where $\mathcal F(g,\xi) :=\{t\in \mathcal F; \cH(t,g)\leq \xi\}$.
\begin{Theorem} [Adapted from the proof of Theorem  7.11  in \cite{massart03}] \label{theoSaintFlour} Under the previous assumptions, let $\hat f $ be
a MLE on $\mathcal F$ defined using a sample $X_{1},\ldots,X_{n}$ of i.i.d. random variables with density $f^\star$. Let $\bar f \in \mathcal F$ such
that $\cH^2(f^\star,\bar f) \leq 2 \inf_{g \in \mathcal F} \cH^2(f^\star,g) $ and let $\xi_{n}$ denotes the unique positive solution of the equation
\begin{equation} \label{eq:WSHyp}
\Psi(\xi_n)=\sqrt{n}\,\xi_n^{2}.
\end{equation}
Then, there exists an absolute constant $\kappa'$ such that, except on a set of probability $\exp(-x)$,
$$ \cH ^2 (f^\star, \hat f)  \leq  \frac 4  {2 \ln 2 -   1 }\kl (f ^\star, \mathcal F) +   \kappa'  \left( \xi_n^2 + \frac x {n} \right)  + (P_n - P)
\left( \frac 1 2  \ln \frac {\bar f }{ f^\star}  \right) $$
where $P$ is the probability measure of density $f^\star$ and $P_n$ is the empirical measure for the observations $X_{1},\ldots,X_{n}$.
\end{Theorem}
Since $(P_n - P) \left( \frac 1 2  \ln \frac {\bar f }{ f^\star}  \right)$ is centered at expectation,  by integrating this tail bound  we find that
\begin{equation} \label{eq:boundHelling}
\E_{f \star} \cH ^2 (f^\star, \hat f)  \leq  \frac 4  {2 \ln 2  -   1 }\kl (f ^\star, \mathcal F) +   \kappa'  \left( \xi_n^2 + \frac 1 {n} \right) 
. 
\end{equation}

For a fixed shape, a configuration $\eta$ and the bounding parameters $\bar{\mu}$, $\underline{\sigma} < 1 <  \bar{\sigma}$, $\underline{L}  <
\bar{L}$, remember that  $\mathcal F_\eta $ is  the set of Gaussian mixture densities parametrized by $\Theta_r(\eta)$.
The following control of the bracketing entropy  can be found in \cite{CohenLepennec11}.
\begin{Lemma} \label{lem:EntropyGMM}
 For a fixed GMM  shape and for all $u \in (0,1)$,
$$ \mbox{H}_{[.]}(\frac u 9, \mathcal F_\eta  ,\cH) \leq  \mathcal I (\eta) + D(\eta)  \ln \frac 1 u $$
where $\mathcal I$ is a constant depending on the GMM shape and the bounding parameters. Moreover, for all $\xi>0$,
\begin{equation} \label{eq:Dudley}
\int_0^{\xi} \sqrt{\mathcal{H}_{[.]}(u,\mathcal F_\eta,\cH )}\,d u \leq  \Psi_\eta(\xi) :=  \xi \sqrt{D(\eta)}
\left\{\mathcal{A}+\sqrt{\ln\left(\frac{1}{1\wedge \xi }\right)}\right\} 
\end{equation}
where the constant $\mathcal{A}$ depends on the GMM shape and the bounding parameters.
\end{Lemma}
We are now in position to finish the proof of Theorem~\ref{theo:BorneSupAvecMLE}. Remember that the sample $\bbX_2$ is used for computing the
maximum likelihood estimators. For a fixed shape, and a given $\eta \in \Ups$, let $ \xi_{n_2}$ satisfying (\ref{eq:WSHyp}) : $\Psi_\eta
(\xi_{n_2})=\sqrt{{n_2}} \,\xi_{n_2}^{2}.$ Note that $ \sqrt{\frac{D(\eta)}{n_2}} \, \mathcal{A} \leq \xi_{n_2} $, thus we have 
\begin{equation*}
 \xi_{n_2}^2 \leq   \frac{D(\eta)}{n_2}   \left\{ 2 \mathcal{A} ^2  +  2 \ln^+ \left( \frac {n_2} {\mathcal{A} ^2   D(\eta)} \right) 
\right\}.
 \end{equation*}
Finally, using Theorem~\ref{thm:main-oracle-inequality}, Lemma~\ref{lem:normsup} and  Inequalities (\ref{eq:boundHelling}) and
(\ref{eq:klHelling}), we find that
\begin{multline*}
    \E_{\bbX_1 \bbX_2} \cH^2(f^*,  \E_{\eta \sim \hat \pi_\lambda}    f_{\hat \theta(\eta)}) \leq 
    c_\lambda  \inf_{\substack{K \in \N^* \\ S   \subset \{ 1, \ldots, d \}}} \bigg\{  \frac{ \ln K! + 1 + 2 |S|  \ln ( e  d / |S|) )}{n_1} \\ 
   +      \lambda C \left[   \kl (f ^\star,
\mathcal F_\eta) +   \kappa    \left( \frac{D(\eta)}{n_2}   \left\{ \mathcal{A} ^2 +  \ln^+ \left( \frac {n_2} {\mathcal{A} ^2   D(\eta)} \right) 
\right\} + \frac 1 {n_2} \right) \right]         \bigg\}
\end{multline*}
where 
$C  =   \frac {8   }  {2 \ln 2  -   1}   \left( 1 + \ln ( \|f^\star\|_\infty   L ^ +)    +   \frac {2 {\bar \mu}^2} {\overline{\sigma} ^2 }  \right)$
and $\kappa = \kappa '  \frac {2 \ln 2  -   1 } {2   }. $
\section*{Acknowledgements}

The authors wish to thank , C. Meynet and C. Maugis for helpful discussions.

\bibliographystyle{plain}
\bibliography{biblio}
\end{document}